 \pdfoutput=1

\documentclass[onefignum,onetabnum]{siamart190516}

\PassOptionsToPackage{hyphens}{url}

\usepackage{ifthen}
\newboolean{cameraready}
\setboolean{cameraready}{true}   

\usepackage{amssymb}
\usepackage{amsmath}
\usepackage{amsfonts}
\usepackage{mathrsfs}

\usepackage{ifxetex}

\ifxetex
\else
	\usepackage[T1]{fontenc}
	\usepackage[utf8]{inputenc}
\fi

\usepackage[setbb,setcolon]{kmath} 			
\usepackage[noadjust]{cite}
\usepackage{soul}

\usepackage{graphicx}
\usepackage{color}
\usepackage{pgfplots}
\pgfplotsset{compat=1.15}

\Crefname{algocf}{Algorithm}{Algorithms}

\def\0{{\mathbf 0}}
\def\1{{\mathbf 1}}

\def\ie{\textit{i.e.},\xspace}
\def\eg{\textit{e.g.},\xspace}
\def\etal{\textit{et al.}\xspace }

\newkmacro\sign[1][{\cdot}]{\mathrm{sign}\left({#1}\right)}
\newkmacro\interior[1][{\cdot}]{\mathrm{int}\left({#1}\right)}
\newkmacro\bd[1][{\cdot}]{\mathrm{bd}\left({#1}\right)}
\newkmacro\conv[1][{\cdot}]{\mathrm{conv}\left({#1}\right)}
\newkmacro\card[1][{\cdot}]{\mathrm{card}\left({#1}\right)}
\newkmacro\spa[1][{\cdot}]{\mathrm{span}\left({#1}\right)}
\newkmacro\nul[1][{\cdot}]{\mathrm{null}\left({#1}\right)}
\newkmacro\deter[1][{\cdot}]{\mathrm{det}\left({#1}\right)}
\newkmacro\tr[1][{\cdot}]{\mathrm{tr}\left({#1}\right)}

\foreach \x in {a,...,z}{
		\expandafter\xdef\csname bf\x \endcsname{\noexpand\ensuremath{\noexpand\mathbf{\x}}}
		\expandafter\xdef\csname bs\x \endcsname{\noexpand\ensuremath{\noexpand\boldsymbol{\x}}}
	}

\foreach \x in {A,...,Z}{
		\expandafter\xdef\csname bs\x \endcsname{\noexpand\ensuremath{\noexpand\boldsymbol{\x}}}
		\expandafter\xdef\csname bf\x \endcsname{\noexpand\ensuremath{\noexpand\mathbf{\x}}}
		\expandafter\xdef\csname bb\x \endcsname{\noexpand\ensuremath{\noexpand\mathbb{\x}}}
		\expandafter\xdef\csname ds\x \endcsname{\noexpand\ensuremath{\noexpand\mathds{\x}}}
		\expandafter\xdef\csname cal\x \endcsname{\noexpand\ensuremath{\noexpand\mathcal{\x}}}
	}

\usepackage{stmaryrd}
\usepackage[squaren,Gray]{SIunits}

\usepackage[mode=buildnew]{standalone}

\usepackage{enumitem}

\usepackage{lipsum}
\usepackage{graphicx}
\usepackage{epstopdf}
\usepackage{algorithmic}
\ifpdf
	\DeclareGraphicsExtensions{.eps,.pdf,.png,.jpg}
\else
	\DeclareGraphicsExtensions{.eps}
\fi

\newcommand{\entry}[1]{(#1)}
\newcommand{\sortedentry}[1]{[#1]}
\newcommand{\column}[1]{#1}

\newcommand{\adjustedvbar}{\big\vert}

\newcommand{\obsletter}{y}
\newcommand{\obs}{\mathbf{\obsletter}}

\def\obsdim{{m}}

\def\tarvec{{\bfx}}

\def\tarvecdim{{n}}

\newcommand{\idxtarel}{j}
\newcommand{\idxtarelsorted}{k}
\newcommand{\idxps}{q}		
\newcommand{\idximp}{p}     

\newcommand{\dicomat}{\bfA}
\newcommand{\atom}{\bfa}

\def\primfun{{P}}								
\def\dualfun{{D}}								
\def\regfun{{r}}								

\def\regslope{{\regfun}_{\textnormal{\scshape slope}}}		
\def\regoscar{{\regfun}_{\textnormal{\scshape oscar}}}		
\newcommand{\dualregslope}{{\regfun}_{\textnormal{\scshape slope},*}}		
\def\slopeweight{{\gamma}}						

\def\pvletter{{x}} 								
\def\pv{{\mathbf{\pvletter}}}					
\def\pvred{{\mathbf{z}}}						
\def\pvdim{{n}}									
\def\pvopt{\pv^\star} 							

\def\dvletter{{u}}								
\def\dv{{\mathbf{\dvletter}}}					
\def\dvopt{{\dv^\star}}							
\def\dfset{{\mathcal{U}}}						

\newcommand{\subdiffletter}{g}					  
\newcommand{\subdiffvec}{\mathbf{\subdiffletter}} 

\newcommand{\saferegion}{\calR}
\newcommand{\idxscreen}{\ell}
\newcommand{\setscreen}{\calL}
\newcommand{\nscreen}{L}      

\newcommand{\sphereregion}{\calS}
\newcommand{\spherec}{\mathbf{c}}
\newcommand{\spherer}{R}

\newcommand{\screenthres}{\kappa}
\newcommand{\upperbound}{B}
\newcommand{\ninequ}{T}
\newcommand{\setinequ}{\mathcal{T}}
\newcommand{\thresinequ}{\tau}
\newcommand{\idxinequ}{t}

\newcommand{\screeningboundfunc}{g}

\newcommand{\atomicSet}{\mathcal{V}}
\newcommand{\atomicsignvec}{\bfs}

\newcommand{\idxit}{t}

\newcommand{\intervint}[2]{\llbracket{#1,#2\rrbracket}}

\newcommand{\reffunctext}[1]{\(#1\)}

\newcommand{\slopeoscar}{\slopeweight}
\newcommand{\oscarparam}{\beta}

\newcommand{\oscar}[1]{\texttt{OSCAR-#1}}

\newcommand{\testA}{{\texttt{test-\idximp=1}}}
\newcommand{\testB}{{\texttt{test-\idximp=\idxps}}}
\newcommand{\testC}{{\texttt{test-all}}}

\newcommand{\algoSlope}{\texttt{PG-no}}
\newcommand{\algoSlopeBao}{\texttt{PG-Bao}}
\newcommand{\algoSlopeScreening}{\texttt{PG-\idximp=\idxps}}
\newcommand{\algoSlopeScreeningp}{\texttt{PG-all}}

\newcommand{\xpnbrep}{50}

\newcommand{\xpm}{100}
\newcommand{\xpn}{300}

\newcommand{\nbRun}{I}
\newcommand{\idRun}{i}
\newcommand{\idSolver}{\texttt{solv}}
\newcommand{\xpgapvar}{\delta}

\newcommand{\xpscreeningitparam}{20}

\newsiamremark{remark}{Remark}
\newsiamremark{hypothesis}{Hypothesis}
\crefname{hypothesis}{Hypothesis}{Hypotheses}
\newsiamthm{claim}{Claim}

\headers{Safe screening rules for the SLOPE problem}{Clément Elvira and Cédric Herzet}

\title{
	Safe rules for the identification of zeros \\ in the solutions of the SLOPE problem\thanks{
		Submitted to the editors DATE. \\
		The research presented in this paper is reproducible. Code and data are
		available at~\url{https://gitlab-research.centralesupelec.fr/2020elvirac/slope-screening}
	}
}

\author{Clément Elvira\thanks{
		IETR UMR CNRS 6164, CentraleSupelec Rennes Campus, 35576 Cesson Sévigné, France
		(\email{clement.elvira@centralesupelec.fr}, \url{https://c-elvira.github.io/}).}
	\and Cédric Herzet\thanks{Inria centre Rennes - Bretagne Atlantique, Rennes, France
		(\email{cedric.herzet@inria.fr}, \url{http://people.rennes.inria.fr/Cedric.Herzet/}).}
}

\usepackage{amsopn}

\usepackage{comment}

\ifpdf
\hypersetup{
  pdftitle={Safe rule for the Slope problem (V2)},
  pdfauthor={Clément Elvira and Cédric Herzet}
}
\fi

\usepackage[]{cleveref} 
\providecommand*{\fullref}[1]{\hyperref[{#1}]{\cref*{#1}. \nameref*{#1}}}
\providecommand*{\Fullref}[1]{\hyperref[{#1}]{\Cref*{#1}. \nameref*{#1}}}

\ifthenelse{\boolean{cameraready}}{
}{
	\makeindex
	\usepackage{makeidx}

}

\begin{document}

\ifthenelse{\boolean{cameraready}}{
}{
	\printindex
}

\maketitle

\begin{abstract}
	In this paper we propose a methodology to accelerate the resolution of the so-called ``Sorted L-One Penalized Estimation'' (SLOPE) problem.
	Our method leverages the concept of ``safe screening'', well-studied in the literature for \textit{group-separable} sparsity-inducing norms, and aims at identifying the zeros in the solution of SLOPE.
	More specifically, we derive a set of \(\tfrac{\pvdim(\pvdim+1)}{2}\) inequalities for each element of the \(\pvdim\)-dimensional primal vector and prove that the latter can be safely screened if some subsets of these inequalities are verified.
	We propose moreover an efficient algorithm to jointly apply the proposed procedure to all the primal variables.
	Our procedure has a complexity \(\mathcal{O}(\tarvecdim\log\tarvecdim+\nscreen\ninequ)\) where \(\ninequ\leq \tarvecdim\) is a problem-dependent constant and \(\nscreen\) is the number of zeros identified by the test.
	Numerical experiments confirm that, for a prescribed computational budget, the proposed methodology leads to significant improvements of the solving precision.
\end{abstract}




\section{Introduction}

During the last decades, sparse linear regression has attracted much attention in the field of statistics, machine learning and inverse problems.
It consists in finding an approximation of some input vector \(\obs\in\kR^\obsdim\) as the linear combination of a few columns of a matrix \(\dicomat\in\kR^{\obsdim\times\tarvecdim}\) (often called dictionary).
Unfortunately, the general form of this problem is NP-hard and convex relaxations have been proposed in the literature to circumvent this issue.
The most popular instance of convex relaxation for sparse linear regression is undoubtedly the so-called ``LASSO'' problem where the coefficients of the regression are penalized by an \(\ell_1\) norm, see~\cite{Chen_siam99}.
Generalized versions of LASSO have also been introduced to account for some possible structure in the pattern of the nonzero coefficients of the regression, see~\cite{Bach:2012re}.

In this paper,  we focus on the following generalization of LASSO:
\begin{equation}
	\label{eq:primal problem}
	\min_{\scriptstyle\tarvec\in\kR^\tarvecdim} \,
	\primfun(\tarvec)
	\triangleq
	\tfrac{1}{2} \kvvbar{\obs - \dicomat\tarvec}_2^2
	+ \lambda\, \regslope(\tarvec),
	\quad \lambda>0
\end{equation}
where
\begin{equation}
	\label{eq:def SLOPE reg}
	\regslope(\tarvec)
	\triangleq
	\sum_{\idxtarelsorted=1}^{\pvdim} \slopeweight_\idxtarelsorted |\pv|_{\sortedentry{\idxtarelsorted}}
\end{equation}
with
\begin{equation}
	\label{eq:hyp slopeweigths}
	\begin{array}{ccc}
		\slopeweight_1 > 0, &  & \slopeweight_1 \geq \dots \geq \slopeweight_\tarvecdim \geq 0,
	\end{array}
\end{equation}
and \(|\tarvec|_{\sortedentry{\idxtarelsorted}}\) is the \(\idxtarelsorted\)th largest element of \(\tarvec\) in absolute value, that is
\begin{equation}
	\forall \tarvec\in\kR^\tarvecdim:\ |\tarvec|_{\sortedentry{1}}\geq |\tarvec|_{\sortedentry{2}} \geq \ldots \geq |\tarvec|_{\sortedentry{\tarvecdim}}
	.
\end{equation}

Problem~\eqref{eq:primal problem} is commonly referred to as \textit{``Sorted L-One Penalized Estimation'' (SLOPE)}  or \textit{``Ordered Weighted L-One Linear Regression''} in the literature and has been introduced in two parallel works \cite{Bogdan2015,Figueiredo2016Ordered}.\footnote{We will stick to the former denomination in the following.}
The first instance of a problem of the form~\eqref{eq:primal problem} (for some nontrivial choice of the parameters \(\slopeweight_\idxtarelsorted\)'s) is due to Bondell and Reich in \cite{Bondell2007}.
The authors considered a problem similar to \eqref{eq:primal problem}, named ``Octagonal Shrinkage and Clustering Algorithm for Regression'' (OSCAR), where the regularization function is a linear combination of an \(\ell_1\) norm and a sum of pairwise \(\ell_\infty\) norms of the elements of \(\tarvec\), that is
\begin{equation}
	\label{eq:def regularzation oscar}
	\regoscar(\tarvec) = \oscarparam_1 \|\tarvec\|_1 + \oscarparam_2 \sum_{\idxtarel' > \idxtarel} \max(|\pv_{\entry{\idxtarel'}}|, |\pv_{\entry{\idxtarel}}|),
\end{equation}
for some \(\oscarparam_1\in\kR*+\), \(\oscarparam_2\in\kR+\).
It is not difficult to see that \reffunctext{\regoscar} can be expressed as a particular case of \reffunctext{\regslope} with the following choice \(\slopeoscar_\idxtarelsorted = \oscarparam_1 + \oscarparam_2 (\tarvecdim-\idxtarelsorted)\).
We note that some authors have recently considered ``group'' versions of the SLOPE problem where the ordered \(\ell_2\) norm of subsets of \(\tarvec\) is penalized by a decreasing sequence of parameters \(\slopeweight_\idxtarelsorted\), see \eg \cite{Grossmann:2015Iden,Gossmann2018Sparse,Brzyski2019dq}.

SLOPE enjoys several desirable properties which have attracted many researchers during the last decade.
First, it was shown in several works that, for some proper choices of parameters \(\slopeweight_\idxtarelsorted\)'s, SLOPE promotes \textit{sparse} solutions with some form of \textit{``clustering''}\footnote{More specifically, groups of nonzero coefficients tend to take on the same value.} of the nonzero coefficients, see \eg \cite{Bondell2007,Figueiredo2016Ordered,Kremer2019fi,schneider2020geometry}.
This feature has been exploited in many application domains:  portfolio optimization \cite{Xing2014lu,Kremer2020hi}, genetics \cite{Grossmann:2015Iden}, magnetic-resonance imaging \cite{elgueddari:hal-02292372}, subspace clustering \cite{Oswal2018Scalable},  deep neural networks \cite{Zhang2018learning}, etc.
Moreover, it has been pointed out in a series of works that SLOPE has very good statistical properties:
it leads to an improvement of the false detection rate (as compared to LASSO)  for moderately-correlated dictionaries \cite{Bogdan2013statistical,Gossmann2018Sparse} and is minimax optimal in some asymptotic regimes, see \cite{Su2016jh,Lecue2017reg}.

Another desirable feature of SLOPE is its convexity.
In particular, it was shown in \cite[Proposition 1.1]{Bogdan2013statistical} and \cite[Lemma 2]{Zeng:2014AtomicNorm} that \reffunctext{\regslope} is a norm as soon as \eqref{eq:hyp slopeweigths} holds.
As a consequence, several numerical procedures have been proposed in the literature to find the global minimizer(s) of problem \eqref{eq:primal problem}.
In \cite{Bogdan2013statistical} and \cite{Zhong2011icml}, the authors considered an accelerated  gradient proximal implementation for SLOPE and OSCAR, respectively.
In \cite{Kremer2020hi}, the authors tackled problem \eqref{eq:primal problem} via an alternating-direction method of multipliers \cite{Boyd2017}.
An approach based on an augmented Lagrangian method was considered in \cite{Luo2019kr}.
In \cite{Zeng:2014AtomicNorm}, the authors expressed \reffunctext{\regslope} as an atomic norm and particularized a Frank-Wolfe minimization procedure \cite{Frank_1956} to problem \eqref{eq:primal problem}.
An efficient algorithm to compute the Euclidean projection onto the unit ball of the SLOPE norm was provided in \cite{Davis2015onlogn}.
Finally, in \cite{Bu2019Algo} a heuristic ``message-passing'' method was proposed.

In this paper, we introduce a new ``safe screening'' procedure to accelerate the resolution of SLOPE.
The concept of ``safe screening'' is well known in the LASSO literature:
it consists in performing simple tests to identify the zero elements of the minimizers;
this knowledge can then be exploited to reduce the problem dimensionality by discarding the columns of the dictionary weighted by the zero coefficients.
Safe screening for LASSO
has been first introduced by El Ghaoui \etal in the seminal paper \cite{Ghaoui2010} and extended to \textit{group-separable} sparsity-inducing norm in \cite{Ndiaye2017}.
Safe screening has rapidly been recognized as a simple yet effective procedure to accelerate the resolution of LASSO, see \eg \cite{fercoq2015, Dai2012Ellipsoid, Xiang:2017ty, icml2014c2_liuc14,wang2015lasso, Herzet16Screening,Herzet:2019fj,guyard2021screen,Le2022HolderDomeTechreport}.
The term ``safe''  refers to the fact that all the elements identified by a safe screening procedure are theoretically guaranteed to correspond to zeros of the minimizers.
In contrast, \textit{unsafe} versions of screening for LASSO (often called ``strong screening rules'') also exist, see \cite{RSSB:RSSB1004}.
More recently, screening methodologies have been extended to detect saturated components in different convex optimization problems, see \cite{Elvira:2020rv,Elvira2020:Squeezing}.

In this paper, we derive \textit{safe} screening rules for SLOPE and emphasize that their implementation enables significant improvements of the solving precision when addressing SLOPE with a prescribed computational budget.
We note that the SLOPE norm is not group-separable and the methodology proposed in \cite{Ndiaye2017} does therefore not trivially apply here.
Prior to this work, we identified two contributions addressing screening for SLOPE.
In~\cite{Larsson2020strong}, the authors proposed an extension of the \textit{strong} screening rules derived in~\cite{RSSB:RSSB1004} to the SLOPE problem.
In~\cite{Bao:2020dq}, the authors suggested a simple test to identify some zeros of the SLOPE solutions.
Although the derivations made by these authors have been shown to contain several technical flaws~\cite{Elvira2021techreport}, their test can be cast as a particular case of our result in \Cref{th: safe screening for SLOPE} (and is therefore quite unexpectedly safe).

The paper is organized as follows. We introduce the notational conventions used throughout the paper in \Cref{sec:notations} and recall the main concepts of safe screening for LASSO in \Cref{sec:Screening: main concepts}.
\Cref{sec:screening-rule} contains our proposed safe screening rules for SLOPE.
\Cref{sec:simus} illustrates the effectiveness of the proposed approach through numerical simulations.
All technical details and mathematical derivations are postponed to \Cref{app:Miscellaneous results,sec:app:main-proofs}.\\


\section{Notations} \label{sec:notations}
Unless otherwise specified, we will use the following conventions throughout the paper.
Vectors are denoted by lowercase bold letters (\eg \(\tarvec\)) and matrices by uppercase bold letters (\eg \(\dicomat\)).
The ``all-zero'' vector of dimension \(\tarvecdim\) is written \({\bf0}_\tarvecdim\). We use symbol \(\ktranspose{}\) to denote the transpose of a vector or a matrix.
\(\tarvec_{\entry{\idxtarel}}\) refers to the \(\idxtarel\)th component of \(\tarvec\).
When referring to the sorted entries of a vector, we use bracket subscripts; more precisely, the notation \(\pv_{\sortedentry{\idxtarelsorted}}\) refers to the \(\idxtarelsorted\)th largest value of \(\pv\).
For matrices, we use \(\atom_{\column{\idxtarel}}\) to denote the \(\idxtarel\)th column of \(\dicomat\).
We use the notation \(|\tarvec|\) to denote the vector made up of the absolute value of the components of \(\tarvec\).
The sign function is defined for all scalars \(x\) as \(\sign[x]=x / \kvbar{x}\) with the convention \(\sign[x]=0\).
%
Calligraphic letters are used to denote sets (\eg \(\calJ\)) and \(\card\) refers to their cardinality.
If \(a<b\) are two integers, \(\intervint{a}{b}\) is used as a shorthand notation for the set \(\{a, a+1,\dotsc,b\}\).
%
Given a vector \(\tarvec\in\kR^{\tarvecdim}\) and a set of indices \(\calJ\subseteq\intervint{1}{\tarvecdim}\), we let \(\tarvec_\calJ\) be the vector of components of \(\tarvec\) with indices in \(\calJ\).
Similarly, \(\dicomat_\calJ\) denotes the submatrix of \(\dicomat\) whose columns have indices in \(\calJ\). \({\dicomat}_{\backslash \idxscreen}\) corresponds to matrix \(\dicomat\) deprived of its \(\ell\)th column.\\

\section{Screening: main concepts}\label{sec:Screening: main concepts}

``Safe screening'' has been introduced by El Ghaoui \etal in \cite{Ghaoui2010} for \(\ell_1\)-penalized problems:
\begin{equation}\label{eq:Ghaoui_GenericProblem}
	\min_{\scriptstyle\tarvec\in\kR^\tarvecdim} \,
	\primfun(\tarvec)
	\triangleq
	f(\dicomat\tarvec)
	+ \lambda\, \|\tarvec\|_1,
	\quad \lambda>0
\end{equation}
where \(\kfuncdef{f}{\kR^\obsdim}{\kR}\) is a closed convex function.
It is grounded on the following ideas.

First, it is well-known that \(\ell_1\)-regularization favors sparsity of the minimizers of \eqref{eq:Ghaoui_GenericProblem}.
For instance, if \(f=\tfrac{1}{2}\|\cdot\|_2^2\) and the solution of \eqref{eq:Ghaoui_GenericProblem} is unique, it can be shown that the minimizer contains at most \(\obsdim\) nonzero coefficients, see \eg \cite[Theorem 3.1]{Foucart2013Mathematical}.
Second, if some zeros of the minimizers are identified, \eqref{eq:Ghaoui_GenericProblem} can be shown to be equivalent to a problem of \textit{reduced} dimension.
More precisely, let \(\setscreen\subseteq\intervint{1}{\tarvecdim}\) be a set of indices such that we have for any minimizer \(\pvopt\) of \eqref{eq:Ghaoui_GenericProblem}:
\begin{equation} \label{eq:hyp zero set}
	\forall \idxscreen\in\setscreen:\ \pvopt_{\entry{\idxscreen}} = 0
\end{equation}
and let \(\bar{\setscreen}=\intervint{1}{\tarvecdim}\backslash \setscreen\).
Then the following problem
\begin{equation}\label{eq:Ghaoui_GenericProblem_reduced}
	\min_{\pvred\in\kR^{\card[\bar{\setscreen}]}} \,
	f(\dicomat_{\bar{\setscreen}}\pvred)
	+ \lambda\, \|\pvred\|_1,
	\quad \lambda>0
\end{equation}
admits the same optimal value as \eqref{eq:Ghaoui_GenericProblem}
and there exists a simple bijection between the minimizers of \eqref{eq:Ghaoui_GenericProblem} and \eqref{eq:Ghaoui_GenericProblem_reduced}.
We note that \(\pv\) belongs to an \(\tarvecdim\)-dimensional space whereas \(\pvred\) is a \(\mathrm{card}(\bar{\setscreen})\)-dimensional vector.
Hence, solving \eqref{eq:Ghaoui_GenericProblem_reduced} rather than \eqref{eq:Ghaoui_GenericProblem} may lead to dramatic memory and computational savings if \(\mathrm{card}(\bar{\setscreen})\ll\pvdim\).

The crux of screening consists therefore in identifying (some) zeros of the minimizers of \eqref{eq:Ghaoui_GenericProblem} with marginal cost.
El Ghaoui \etal emphasized that this is possible by relaxing some primal-dual optimality condition of problem \eqref{eq:Ghaoui_GenericProblem}.
More precisely, let
\begin{equation}\label{eq:dual lasso f*}
	\dvopt \in
	\kargmax_{\dv \in\kR^\obsdim}\;
	\dualfun(\dv)
	\triangleq
	-f^*(-\dv)
	\quad
	\mbox{\(\mathrm{s.t.} \quad \|\ktranspose{\dicomat}\dv\|_\infty\leq \lambda\)}
\end{equation}
be the dual problem of \eqref{eq:Ghaoui_GenericProblem}, where \reffunctext{f^*} denotes the Fenchel conjugate.
Then, by complementary slackness, we must have for any minimizer \(\pvopt\) of \eqref{eq:Ghaoui_GenericProblem}:
\begin{equation}\label{eq:ideal screening Lasso from KKT}
	\forall \idxscreen\in\intervint{1}{\tarvecdim}:\
	(|\ktranspose{\atom}_{\column{\idxscreen}}\dvopt|-\lambda)\,\pvopt_{\entry{\idxscreen}}=0.
\end{equation}
Since dual feasibility imposes that \(|\ktranspose{\atom}_{\column{\idxscreen}}\dvopt|\leq \lambda\), we obtain the following implication:
\begin{equation}\label{eq:ideal screening Lasso}
	|\ktranspose{\atom}_{\column{\idxscreen}}\dvopt|<\lambda
	\implies
	\pvopt_{\entry{\idxscreen}}=0
	.
\end{equation}
Hence, if \(\dvopt\) is available, the left-hand side of \eqref{eq:ideal screening Lasso} can be used  to detect if the \(\idxscreen\)th component of \(\pvopt\) is equal to zero.

Unfortunately, finding a maximizer of dual problem \eqref{eq:dual lasso f*} is generally as difficult as solving primal problem \eqref{eq:Ghaoui_GenericProblem}.
This issue can nevertheless be circumvented by identifying some region \(\saferegion\) of the dual space (commonly referred to as \textit{``safe region''}) such that \(\dvopt\in\saferegion\).
Indeed, since
\begin{equation}\label{eq:relaxed screening}
	\max_{\dv\in\saferegion} \ |\ktranspose{\atom}_{\column{\idxscreen}}\dv|
	<
	\lambda
	\implies
	|\ktranspose{\atom}_{\column{\idxscreen}}\dvopt|
	<
	\lambda
	,
\end{equation}
the left-hand side of  \eqref{eq:relaxed screening} constitutes an alternative (weaker) test to detect the zeros of \(\pvopt\).
For proper choices of \(\saferegion\), the maximization over \(\dv\) admits a simple analytical solution.
For example, if \(\saferegion\) is a ball, that is
\begin{equation}\label{eq:definition  sphere}
	\saferegion
	=
	\sphereregion(\spherec,\spherer)
	\triangleq
	\kset{\dv\in\kR^\obsdim}{\|\dv-\spherec\|_2\leq \spherer},
\end{equation}
then \(\max_{\dv\in\saferegion} |\ktranspose{\atom}_{\column{\idxscreen}}\dv| = |\ktranspose{\atom}_{\column{\idxscreen}}\spherec| +\spherer\kvvbar{\atom_{\column{\idxscreen}}}_2\) and the relaxation of \eqref{eq:relaxed screening} leads to
\begin{equation}\label{eq:sphere screening Lasso}
	|\ktranspose{\atom}_{\column{\idxscreen}}\spherec|
	<
	\lambda - \spherer \kvvbar{\atom_{\column{\idxscreen}}}_2 \implies \pvopt_{\entry{\idxscreen}}=0
	.
\end{equation}
In this case, the screening test is straightforward to implement since it only requires the  evaluation of one inner product between \(\atom_{\column{\idxscreen}}\) and \(\spherec\).\footnote{We note that the \(\ell_2\)-norm appearing in the expression of the test is usually considered as ``known'' since it can be evaluated offline.
}

Many procedures have been proposed in the literature to construct safe spheres \cite{NIPS2011_0578,fercoq2015,Ndiaye2017} or safe regions with refined geometries \cite{Xiang2012Fast,Xiang:2017ty,Dai2012Ellipsoid,Le2022HolderDomeTechreport}.
If \reffunctext{f^*} is a \(\zeta\)-strongly convex function, a popular approach to construct a safe region is the so-called ``GAP sphere'' \cite{Ndiaye2017} whose center and radius are defined as follows:
\begin{equation}
	\label{eq:gap sphere}
	\begin{array}{ll}
		\spherec & = \dv                                                  \\
		\spherer & = \sqrt{\tfrac{2}{\zeta}(\primfun(\pv)-\dualfun(\dv))}
	\end{array}
\end{equation}
where \((\pv,\dv)\) is any primal-dual feasible couple.
This approach has gained in popularity because of its good behavior when \((\pv,\dv)\) is close to optimality.
In particular, if \reffunctext{f} is proper lower semi-continuous, \(\pv=\pvopt\) and \(\dv=\dvopt\), then \(\primfun(\pv)-\dualfun(\dv)=0\) by strong duality~\cite[Proposition~15.22]{Bauschke2017}.
In this case, screening test~\eqref{eq:sphere screening Lasso} reduces to~\eqref{eq:ideal screening Lasso} and, except in some degenerated cases, all the zero components of \(\pvopt\) can be identified by the screening test.
Interestingly, this behavior also provably occurs for sufficiently small values of the dual gap~\cite[Propositions~8 and~9]{Ndiaye2021Converging} and has been observed in many numerical experiments, see \eg \cite{fercoq2015,Ndiaye2017,Herzet:2019fj,Elvira2020:Squeezing}.


As a final remark, let us mention that the framework presented in this section extends to optimization problems where the (sparsity-promoting) penalty function describes a group-separable norm, see \textit{e.g.},~\cite{Ndiaye2017,Dantas2021}.
In particular, the complementary slackness condition~\eqref{eq:ideal screening Lasso from KKT} still holds (up to a minor modification), thus allowing to design safe screening tests based on the same rationale.
We note that, since the SLOPE penalization does not feature such a separability property, the methodology presented in this section does unfortunately not apply.
%
\\


\section{Safe screening rules for SLOPE} \label{sec:screening-rule}
In this section, we propose a new procedure to extend the concept of safe screening to SLOPE.
Our exposition is organized as follows.
In \Cref{sec:working hyps} we describe our working assumptions and in \Cref{sec:SLOPE screening main} we present a family of screening tests for SLOPE
(see \Cref{th: safe screening for SLOPE}).
Each test is defined by a set of parameters \(\{\idximp_\idxps\}_{\idxps\in\intervint{1}{\tarvecdim}}\) and takes the form of a series of inequalities.
We show that a simple test of the form \eqref{eq:sphere screening Lasso} can be recovered for some particular values of the parameters \(\{\idximp_\idxps\}_{\idxps\in\intervint{1}{\tarvecdim}}\), although this choice does not correspond to the most effective test in the general case.
In \Cref{sec:implementation SLOPE screening}, we finally propose an efficient numerical procedure to verify simultaneously \textit{all} the proposed screening tests.\\

\subsection{Working hypotheses}\label{sec:working hyps}

In this section, we present two working assumptions which are assumed to hold in the rest of the paper even when not explicitly mentioned.

We first suppose that the regularization parameter \(\lambda\) satisfies
\begin{equation}
	\label{eq:lambda-such-that-0-is-sol}
	0<\lambda < \lambda_{\max} \triangleq \max_{\idxps\in\intervint{1}{\tarvecdim}}\kparen{\sum_{\idxtarelsorted=1}^\idxps \adjustedvbar\ktranspose{\dicomat}\obs\adjustedvbar_{\sortedentry{\idxtarelsorted}}/ \sum_{\idxtarelsorted=1}^\idxps \slopeweight_\idxtarelsorted}
	.
\end{equation}
In particular, the hypothesis \(\lambda_{\max}>0\) is tantamount to assuming that \(\obs\notin\ker(\ktranspose{\dicomat})\).
On the other hand, \(\lambda < \lambda_{\max}\) prevents the vector \(\0_\tarvecdim\) from being a minimizer of the SLOPE problem~\eqref{eq:primal problem}.
More precisely, it can be shown that under condition~\eqref{eq:hyp slopeweigths},
\begin{equation}
	\label{eq:cns solution slope is zero}
	\text{\(\lambda\) and \(\kfamily{\slopeweight_\idxtarelsorted}{\idxtarelsorted=1}^\tarvecdim\) verify~\eqref{eq:lambda-such-that-0-is-sol}}
	\Longleftrightarrow
	\text{\({\bf0}_\tarvecdim\) is not a minimizer of~\eqref{eq:primal problem}.}
\end{equation}
A proof of this result is provided in \Cref{subsec:app:proof cns solutions slope is zero}.

Second, we assume that the columns of the dictionary \(\dicomat\) are unit-norm, \ie
\begin{equation}
	\label{eq:assumption:atom unit norm}
	\forall\idxtarel\in\intervint{1}{\tarvecdim}:
	\quad
	\kvvbar{\atom_{\column{\idxtarel}}}_2 = 1.
\end{equation}
Assumption~\eqref{eq:assumption:atom unit norm} simplifies the statement of our results in the next subsection.
However, all our subsequent derivations can be easily extended to the general case where \eqref{eq:assumption:atom unit norm} does not hold.\\

\subsection{Safe screening rules} \label{sec:SLOPE screening main}

In this section, we derive a family of safe screening rules for SLOPE.

Let us first note that \eqref{eq:primal problem} admits at least one minimizer and our screening problem is therefore well-posed.
Indeed, the primal cost function in \eqref{eq:primal problem} is continuous and coercive since \reffunctext{\regslope} is a norm (see \eg \cite[Proposition 1.1]{Bogdan2013statistical} or \cite[Lemma 2]{Zeng:2014AtomicNorm}); the existence of a minimizer then follows from Weierstrass theorem~\cite[Theorem~1.29]{Bauschke2017}.
In the following, we will assume that the minimizer is unique to simplify our statements.
Nevertheless, all our results extend to the general case where there exist more than one minimizer by replacing ``\(\pvopt_{\entry{\idxscreen}}=0\)'' by ``\(\pvopt_{\entry{\idxscreen}}=0\ \mbox{for any minimizer of \eqref{eq:primal problem}}\)'' in all our subsequent statements.

Our starting point to derive our safe screening rules is the following primal-dual optimality condition:
\begin{theorem}
	\label{th:safe-screening}
	Let
	\begin{equation}
		\label{eq:dual problem}
		\dvopt
		=
		\kargmax_{\dv\in\dfset}\;
		\dualfun(\dv)
		\triangleq
		\tfrac{1}{2}\|\obs\|_2^2 - \tfrac{1}{2}\|\obs - \dv\|_2^2
		,
	\end{equation}
	where
	\begin{equation}
		\label{eq:dual set}
		\dfset = \kset{\dv}{\sum_{\idxtarelsorted=1}^\idxps \adjustedvbar\ktranspose{\dicomat}\dv\adjustedvbar_{\sortedentry{\idxtarelsorted}}\leq \lambda \sum_{\idxtarelsorted=1}^\idxps \slopeweight_\idxtarelsorted,\,\idxps\in\intervint{1}{\tarvecdim}}.
	\end{equation}
	Then, for all integers \(\idxscreen\in\intervint{1}{\tarvecdim}\):
	\begin{equation}
		\label{eq: ideal safe screening test}
		\forall \idxps\in\intervint{1}{\tarvecdim}:
		\  \adjustedvbar\ktranspose{\atom}_{\column{\idxscreen}}\dvopt\adjustedvbar + \sum_{\idxtarelsorted=1}^{\idxps-1} \adjustedvbar\ktranspose{\dicomat}_{\backslash \idxscreen}\dvopt\adjustedvbar_{\sortedentry{\idxtarelsorted}}
		<
		\lambda {\sum_{\idxtarelsorted=1}^{\idxps} \slopeweight_\idxtarelsorted}
		\implies \pv^{\star}_{\entry{\idxscreen}} = 0
		.
	\end{equation}
\end{theorem}
A proof of this result is provided in \Cref{subsec:app:proof ideal screening Slope}.
We mention that, although it differs quite significantly in its formulation, \Cref{th:safe-screening} is closely related to \cite[Proposition~1]{Larsson2020strong}.\footnote{We refer the reader to Section~SM1 of the electronic supplementary material of this paper for a detailed description and a proof of the connection between these two results.}
We also note that \eqref{eq:dual problem} corresponds to the dual problem of~\eqref{eq:primal problem}, see \eg~\cite[Section~2.5]{Bogdan2013statistical}.
Moreover, ${\dvopt}$ exists and is unique because \reffunctext{\dualfun} is a continuous strongly-concave function and \(\dfset\) a closed set.
The equality in~\eqref{eq:dual problem} is therefore well-defined.

\Cref{th:safe-screening} provides a condition similar to \eqref{eq:ideal screening Lasso} relating the dual optimal solution \(\dvopt\) to the zero components of the primal minimizer \(\pvopt\).
Unfortunately, evaluating the dual solution \(\dvopt\) requires a computational load comparable to the one needed to solve the SLOPE problem~\eqref{eq:primal problem}.
Similarly to \(\ell_1\)-penalized problems, tractable screening rules can nevertheless be devised if ``easily-computable'' upper bounds on the left-hand side of \eqref{eq: ideal safe screening test} can be found.
In particular, for any set \(\{\upperbound_{\idxps,\idxscreen}\in\kR\}_{\idxps\in\intervint{1}{\tarvecdim}}\) verifying
\begin{equation}\label{eq:def upper bound}
	\forall \idxps\in\intervint{1}{\tarvecdim}:
	\  \adjustedvbar\ktranspose{\atom}_{\column{\idxscreen}}\dvopt\adjustedvbar + \sum_{\idxtarelsorted=1}^{\idxps-1} \adjustedvbar\ktranspose{\dicomat}_{\backslash \idxscreen}\dvopt\adjustedvbar_{\sortedentry{\idxtarelsorted}}
	\leq
	\upperbound_{\idxps,\idxscreen}
	,
\end{equation}
we readily have that
\begin{equation}
	\label{eq:relaxed SLOPE screening test}
	\forall \idxps\in\intervint{1}{\tarvecdim}:
	\upperbound_{\idxps,\idxscreen}
	<
	\lambda {\sum_{\idxtarelsorted=1}^{\idxps} \slopeweight_\idxtarelsorted}
	\implies \pv^{\star}_{\entry{\idxscreen}} = 0
	.
\end{equation}
The next lemma provides several instances of such upper bounds:\vspace{0.2cm}
\begin{lemma}\label{lemma:upper bound}
	Let \(\dvopt\in\sphereregion(\spherec,\spherer)\).
	Then \(\forall \idxscreen\in\intervint{1}{\tarvecdim}\) and \(\forall \idxps\in\intervint{1}{\tarvecdim}\), we have that
	\begin{equation}
		\nonumber
		\upperbound_{\idxps,\idxscreen}
		\triangleq
		\adjustedvbar
		\ktranspose{\atom}_{\column{\idxscreen}}\spherec
		\adjustedvbar
		+ \sum_{\idxtarelsorted=\idximp}^{\idxps-1} \adjustedvbar{
			\ktranspose{\dicomat}_{\backslash \idxscreen}\spherec
		}\adjustedvbar_{\sortedentry{\idxtarelsorted}}
		+
		(\idxps-\idximp+1)
		\spherer
		+ \lambda\sum_{\idxtarelsorted=1}^{\idximp-1} \slopeweight_{\idxtarelsorted}
	\end{equation}
	verifies \eqref{eq:def upper bound} for any \(\idximp\in\intervint{1}{\idxps}\).\vspace{0.2cm}
\end{lemma}
A proof of this result is available in \Cref{sec:Proof of lemma:upper bound}.
We note that Lemma 4.2 defines \textit{one} particular family of upper bounds on the left-hand side of \eqref{eq:def upper bound}.
The derivation of these upper bounds is based on the knowledge of a  safe spherical region and partially exploits the definition of the dual feasible set, see \Cref{sec:Proof of lemma:upper bound}.
We nevertheless emphasize that other choices of safe regions or majorization techniques can be envisioned and possibly lead to more favorable upper bounds.

Defining
\begin{equation}
	\screenthres_{\idxps,\idximp} \triangleq
	\lambda
	\Bigg(
	\sum_{\idxtarelsorted=\idximp}^{\idxps} \slopeweight_\idxtarelsorted
	\Bigg)
	-
	(\idxps-\idximp+1)\spherer,
\end{equation}
a straightforward particularization of \eqref{eq:relaxed SLOPE screening test} then leads to the following safe screening rules for SLOPE:\vspace{0.1cm}
\begin{theorem}\label{th: safe screening for SLOPE}
	Let \(\{\idximp_\idxps\}_{\idxps\in\intervint{1}{\pvdim}}\) be a sequence such that \(\idximp_\idxps\in\intervint{1}{\idxps}\) for all \(\idxps\in\intervint{1}{\pvdim}\).
	Then, the following statement holds:
	\begin{equation}\label{eq: general safe screening for SLOPE}
		\forall \idxps\in\intervint{1}{\tarvecdim}:
		\adjustedvbar{
		\ktranspose{\atom}_{\column{\idxscreen}}\spherec
		}\adjustedvbar
		+ \sum_{\idxtarelsorted=\idximp_\idxps}^{\idxps-1} \adjustedvbar{
			\ktranspose{\dicomat}_{\backslash \idxscreen}\spherec
		}\adjustedvbar_{\sortedentry{\idxtarelsorted}}
		< \screenthres_{\idxps,\idximp_\idxps}
		\implies \pv^{\star}_{\entry{\idxscreen}} = 0
		.
	\end{equation}
\end{theorem}
{We mention that the notation ``\(\idximp_\idxps\)'' is here introduced to stress the fact that a different value of \(\idximp\) can be used for each \(\idxps\) in~\eqref{eq: general safe screening for SLOPE}.
Since \(\idxps\in\intervint{1}{\tarvecdim}\) and each parameter \(\idximp_\idxps\) can take on \(\idxps\) different values in \Cref{th: safe screening for SLOPE},~\eqref{eq: general safe screening for SLOPE} thus defines \(\tarvecdim!\) different screening tests for SLOPE
where \(\tfrac{\pvdim(\pvdim+1)}{2}\) distinct inequalities are involved.
We discuss two particular choices of parameters \(\{\idximp_\idxps\}_{\idxps\in\intervint{1}{\tarvecdim}}\) below and  propose an efficient procedure to jointly evaluate all the tests defined by feasible sequences \(\{\idximp_\idxps\}_{\idxps\in\intervint{1}{\pvdim}}\) in the next section.

Let us first consider the case where
\begin{equation}\label{eq:rq=1}
	\forall\idxps\in\intervint{1}{\tarvecdim}:\ \idximp_\idxps=1.
\end{equation}
Screening test~\eqref{eq: general safe screening for SLOPE} then particularizes as
\begin{equation}\label{eq: safe screening for SLOPE p=q-1}
	\forall \idxps\in\intervint{1}{\tarvecdim}:\
	\adjustedvbar{\ktranspose{\atom}_{\column{\idxscreen}}\spherec}\adjustedvbar
	+
	\sum_{\idxtarelsorted=1}^{\idxps-1} \adjustedvbar{
		\ktranspose{\dicomat}_{\backslash \idxscreen}\spherec
	}\adjustedvbar_{\sortedentry{\idxtarelsorted}}
	<
	\lambda
	\Bigg(
	\sum_{\idxtarelsorted=1}^{\idxps} \slopeweight_\idxtarelsorted
	\Bigg)
	- \idxps\spherer
	\implies \pv^{\star}_{\entry{\idxscreen}} = 0
	.
\end{equation}
Interestingly, \eqref{eq: safe screening for SLOPE p=q-1} shares the same mathematical structure as optimality condition \eqref{eq: ideal safe screening test}.
In particular, \eqref{eq: safe screening for SLOPE p=q-1} reduces to \eqref{eq: ideal safe screening test} when \(\spherec=\dvopt\) and \(\spherer=0\).
In this case, it is easy to see that \eqref{eq: safe screening for SLOPE p=q-1} is the best\footnote{In the following sense: if test \eqref{eq: general safe screening for SLOPE} passes for some choice of the parameters \(\{\idximp_\idxps\}_{\idxps\in\intervint{1}{\tarvecdim}}\), then test \eqref{eq: safe screening for SLOPE p=q-1} also necessarily succeeds.
\label{footnote:definition of best screening test}} screening test within the family of tests defined in \Cref{th: safe screening for SLOPE} since an equality occurs in \eqref{eq:def upper bound}.

\begin{figure}
	\centering
	\includegraphics[width=.7\columnwidth]{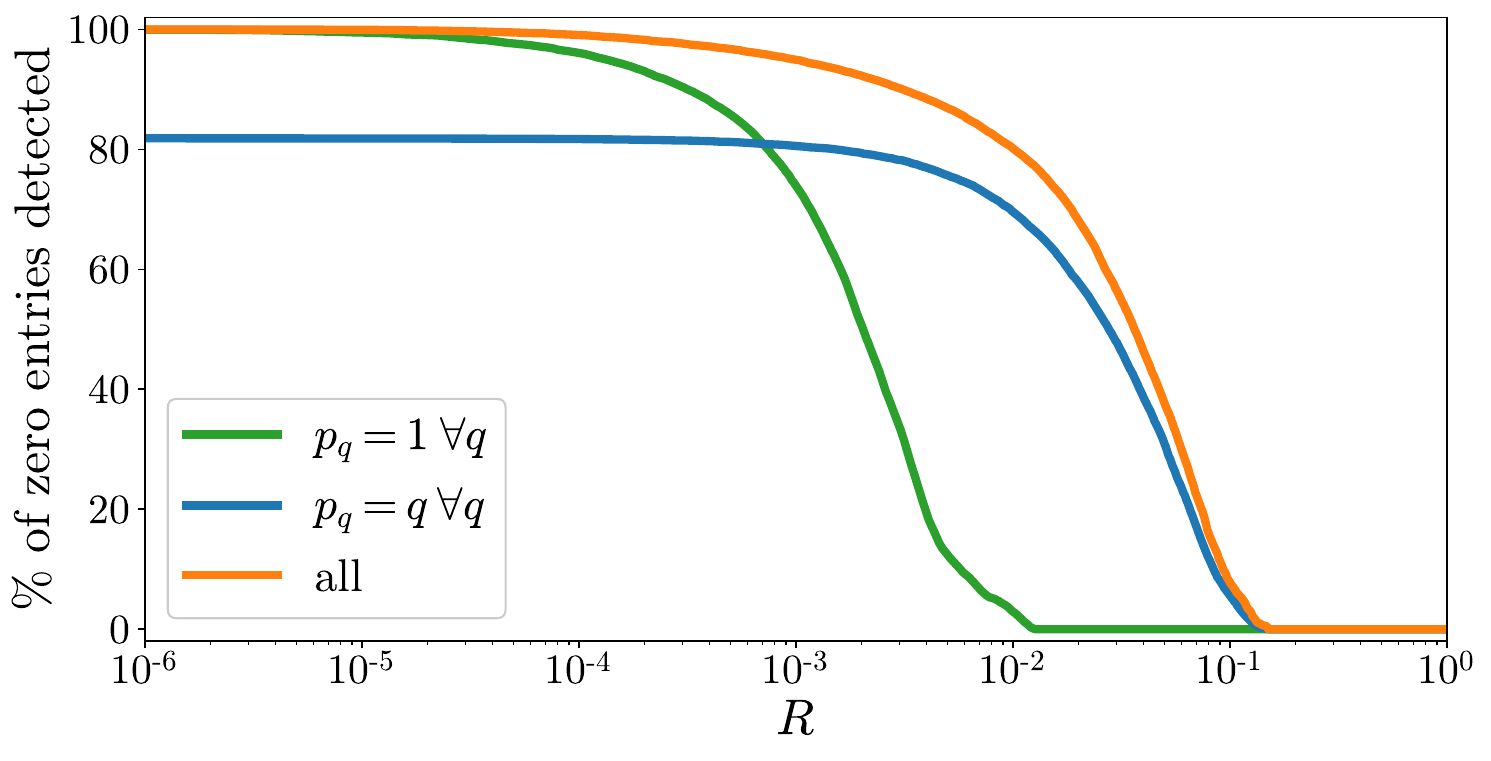}
	\caption{
	\label{fig:comparing choice of q_r}
	Percentage of zero entries in \(\pvopt\) detected by the safe screening tests as a function of \(\spherer\), the radius of the safe sphere.
	Each curve corresponds to a different implementation of the safe screening test~\eqref{eq: general safe screening for SLOPE}: \(\idximp_\idxps=1\) \(\forall \idxps\), see~\eqref{eq: safe screening for SLOPE p=q-1} (green curve), \(\idximp_\idxps=\idxps\) \(\forall \idxps\), see \eqref{eq: safe screening for SLOPE p=0} (blue curve), and all possible choices for \(\{\idximp_\idxps\}_{\idxps\in\intervint{1}{\tarvecdim}}\) (orange curve).
	The results are generated by using the \oscar{1} sequence for \(\{\slopeweight_\idxtarelsorted\}_{\idxtarelsorted=1}^\tarvecdim\), the Toeplitz dictionary and the ratio \(\lambda / \lambda_{\max}=0.5\), see \Cref{subsec:simu:setup}.
	}
\end{figure}

In practice, we may expect this conclusion to remain valid when \(\spherer\) is ``sufficiently'' close to zero. This behavior is illustrated in \Cref{fig:comparing choice of q_r}.
The figure {represents the proportion of zeros entries of \(\pvopt\) detected by screening test \eqref{eq: general safe screening for SLOPE} for different ``qualities'' of the safe region and different choices of parameters \(\{\idximp_\idxps\}_{\idxps\in\intervint{1}{\tarvecdim}}\).
We refer the reader to \Cref{subsec:simu:setup} for a detailed description of the simulation setup.
The center of the safe sphere used to apply \eqref{eq: general safe screening for SLOPE} is assumed to be equal (up to machine precision) to \(\dvopt\) and the \(x\)-axis of the figure represents the radius \(\spherer\) of the sphere region.
The green curve corresponds to test~\eqref{eq: safe screening for SLOPE p=q-1}; the orange curve represents the screening performance achieved when test \eqref{eq: general safe screening for SLOPE} is implemented for all possible choices for \(\{\idximp_\idxps\}_{\idxps\in\intervint{1}{\tarvecdim}}\). We note that, as expected, the green curve attains the best screening performance as soon as \(\spherer\) becomes close to zero.

At the other extreme of the spectrum, another case of interest reads as:
\begin{equation}
	\forall\idxps\in\intervint{1}{\tarvecdim}:\ \idximp_\idxps=\idxps.
\end{equation}
Using our initial hypothesis~\eqref{eq:hyp slopeweigths}, the screening test~\eqref{eq: general safe screening for SLOPE} rewrites\footnote{
\label{footnote:proof case pq=q}
More precisely, \eqref{eq: general safe screening for SLOPE} reduces to ``\(\forall \idxps\in\intervint{1}{\pvdim}:\ |{
\ktranspose{\atom}_{\column{\idxscreen}}\spherec
}|
<
\lambda \slopeweight_\idxps - \spherer
\implies \pv^{\star}_{\entry{\idxscreen}} = 0\)'' which, in view of \eqref{eq:hyp slopeweigths}, is equivalent to \eqref{eq: safe screening for SLOPE p=0}.
}
\begin{equation} \label{eq: safe screening for SLOPE p=0}
	|{
	\ktranspose{\atom}_{\column{\idxscreen}}\spherec
	}|
	<
	\lambda \slopeweight_\tarvecdim - \spherer
	\implies \pv^{\star}_{\entry{\idxscreen}} = 0
	.
\end{equation}
Interestingly, this test has the same mathematical structure as \eqref{eq:sphere screening Lasso} with the exception that \(\lambda\) is multiplied by the value of the smallest weighting coefficient \(\slopeweight_\tarvecdim\).
In particular, if \(\slopeweight_\idxtarelsorted=1\) \(\forall\idxtarelsorted\in\intervint{1}{\pvdim}\) SLOPE reduces to LASSO and test \eqref{eq: safe screening for SLOPE p=0} is equivalent to \eqref{eq:sphere screening Lasso}; \Cref{th: safe screening for SLOPE} thus encompasses standard screening rule \eqref{eq:sphere screening Lasso} for LASSO as a particular case.
The following result emphasizes that \eqref{eq: safe screening for SLOPE p=0} is in fact the best screening rule within the family of tests defined by \Cref{th: safe screening for SLOPE} when \(\slopeweight_\idxtarelsorted=1\) \(\forall\idxtarelsorted\in\intervint{1}{\pvdim}\):
\begin{lemma}\label{lemma:optimality p=0 for LASSO}
	If \(\slopeweight_\idxtarelsorted=1\) \(\forall\idxtarelsorted\in\intervint{1}{\pvdim}\) and test \eqref{eq: general safe screening for SLOPE} passes for some choice of parameters \(\{\idximp_\idxps\}_{\idxps\in\intervint{1}{\tarvecdim}}\), then test \eqref{eq: safe screening for SLOPE p=0} also succeeds.
\end{lemma}
A proof of this result is available in \Cref{sec:Proof of lemma:optimality p=0 for LASSO}.

As a final remark, let us mention that,
although we just emphasized that some choices of parameters $\{\idximp_\idxps\}_{\idxps\in\intervint{1}{\tarvecdim}}$ can be optimal (in terms of screening performance) in some situations, no conclusion can be drawn in the general case.
In particular, we found in our numerical experiments that the best choice for $\{\idximp_\idxps\}_{\idxps\in\intervint{1}{\tarvecdim}}$ depends on many factors: the weights $\{\slopeweight_\idxtarelsorted\}_{\idxtarelsorted=1}^\tarvecdim$, the radius of the safe sphere \(\spherer\), the nature of the dictionary, the atom to screen, etc.
This is illustrated in Fig.~\ref{fig:comparing choice of q_r}:
we see that the blue and green curves deviate from the orange curve for certain values of \(\spherer\), that is the best screening performance is not necessarily achieved for \(\idximp_\idxps=1\) or \(\idximp_\idxps=\idxps \ \forall \idxps\in\intervint{1}{\pvdim}\).\\

\subsection{Efficient implementation}
\label{sec:implementation SLOPE screening}

Since the best values for \(\{\idximp_\idxps\}_{\idxps\in\intervint{1}{\tarvecdim}}\) cannot be foreseen, it is desirable to evaluate the screening rule~\eqref{eq: general safe screening for SLOPE} for \textit{any} choice of these parameters.
Formally, this ideal test reads: 
\begin{equation}
	\label{eq: screening SLOPE all p}
	\forall \idxps\in\intervint{1}{\tarvecdim},\exists \idximp_{{\idxps}}\in\intervint{1}{\idxps}:
	\adjustedvbar
	\ktranspose{\atom}_{\column{\idxscreen}}\spherec
	\adjustedvbar
	+ \sum_{\idxtarelsorted=\idximp_{{\idxps}}}^{\idxps-1} \adjustedvbar
	\ktranspose{\dicomat}_{\backslash \idxscreen}\spherec
	\adjustedvbar_{\sortedentry{\idxtarelsorted}}
	<
	\screenthres_{\idxps,\idximp_{{\idxps}}}
	\implies \pv^{\star}_{\entry{\idxscreen}} = 0
	.
\end{equation}
Since verifying this test for a \textit{given} index $\idxscreen$ involves the evaluation of $\mathcal{O}(\tarvecdim^2)$ inequalities, a brute-force evaluation of \eqref{eq: screening SLOPE all p} for all atoms of the dictionary requires $\mathcal{O}(\tarvecdim^3)$ operations.
In this section, we present a procedure to perform this task with a complexity scaling as $\mathcal{O}(\tarvecdim \log \tarvecdim + \ninequ \nscreen)$ where $\ninequ\leq \tarvecdim$ is some problem-dependent constant (to be defined later on) and $\nscreen$ is the number of atoms of the dictionary passing test \eqref{eq: screening SLOPE all p}.
Our procedure is summarized in \Cref{alg:Implementation generalized SLOPE screening test 2,alg:Implementation generalized SLOPE screening test}, and is grounded on the following nesting properties.\\

\begin{algorithm}[t!]
	\caption{
		\label{alg:Implementation generalized SLOPE screening test 2}
		Fast implementation of SLOPE screening test \eqref{eq: screening SLOPE all p}
	}
	\begin{algorithmic}[1]
		\REQUIRE{radius \(\spherer\geq0\), sorted elements $\{|\ktranspose{\dicomat}\spherec|_{\sortedentry{\idxtarelsorted}}\}_{\idxtarelsorted=1}^\tarvecdim$}
		\STATE $\setscreen = \emptyset$ \COMMENT{Set of screened atoms: init}
		\STATE $\idxscreen = \tarvecdim$ \COMMENT{Index of atom under testing: init}
		\STATE Evaluate $\{\screeningboundfunc({\idximp})\}_{{\idximp}=1}^{\tarvecdim}$,
		$\{\idximp^\star({\idxps})\}_{{\idxps}=1}^{\tarvecdim}$,
		$\{\idxps^\star(\idxtarelsorted)\}_{\idxtarelsorted=1}^{\tarvecdim}$
		\STATE $\mathrm{run} = 1$

		\vspace*{.2em}
		\WHILE{$\mathrm{run} == 1$ and $\ell>0$}
		\STATE $\mathrm{test}$  = Algorithm \ref{alg:Implementation generalized SLOPE screening test}($\spherer$,$\idxscreen$,$\{\screeningboundfunc({\idximp})\}_{{\idximp}=1}^{\tarvecdim}$,$\{\idximp^\star({\idxps})\}_{{\idxps}=1}^{\tarvecdim}$,$\{\idxps^\star(\idxtarelsorted)\}_{\idxtarelsorted=1}^{\tarvecdim}$)

		\IF{$\mathrm{test}==1$} 
		\STATE $\setscreen = \setscreen \cup\{\idxscreen\}$
		\STATE $\ell = \ell-1$
		\ELSE
		\STATE $\mathrm{run} = 0$ \COMMENT{Stop testing as soon as one atom does not pass the test}
		\ENDIF

		\ENDWHILE

		\RETURN{$\setscreen$ (Set of indices passing test \eqref{eq: screening SLOPE all p})}
	\end{algorithmic}
\end{algorithm}

\noindent
\paragraph{Nesting of the tests for different atoms} We first emphasize that there exists an implication between the failures of test \eqref{eq: screening SLOPE all p} for some group of indices.
In particular, the following result holds:\vspace{0.1cm}
\begin{lemma}\label{lemma:nesting test}
	Let $\upperbound_{\idxps,\idxscreen}$ be defined as in \Cref{lemma:upper bound} and assume that
	\begin{equation} \label{eq:WH}
		\kvbar{\ktranspose{\atom}_{\column{1}}\spherec}\geq \ldots \geq \kvbar{\ktranspose{\atom}_{\column{\tarvecdim}}\spherec}.
	\end{equation}
	Then $\forall\idxps\in\intervint{1}{\tarvecdim}$:
	\begin{equation}
		\idxscreen<\idxscreen' \implies \upperbound_{\idxps,\idxscreen}\geq \upperbound_{\idxps,\idxscreen'}
		.
	\end{equation}
\end{lemma}
A proof of this result is provided in \Cref{subsec:proof lemma nesting test}.
\Cref{lemma:nesting test} has the following consequence: if \eqref{eq:WH} holds, the failure of test \eqref{eq: screening SLOPE all p} for some $\idxscreen'\in\intervint{2}{\tarvecdim}$ implies the failure of the test for any index $\idxscreen\in \intervint{1}{\idxscreen'-1}$.
This immediately suggests a backward strategy for the evaluation of \eqref{eq: screening SLOPE all p}, starting from $\idxscreen=\tarvecdim$ and going backward to smaller indices.
This is the sense of the main recursion in \Cref{alg:Implementation generalized SLOPE screening test 2}.

We note that hypothesis~\eqref{eq:WH} can always be verified by a proper reordering of the elements of $|\ktranspose{\dicomat}\spherec$|.
This can be achieved by state-of-the-art sorting procedures with a complexity of $\mathcal{O}(\tarvecdim \log \tarvecdim)$.
Therefore, in the sequel we will assume that~\eqref{eq:WH} holds even if not explicitly mentioned.\\

\paragraph{Nesting of some inequalities}
We next show that the number of inequalities to be verified may possibly be substantially smaller than $\mathcal{O}(\tarvecdim^2)$.
We first focus on the case ``$\idxscreen=\tarvecdim$'' and then extend our result to the general case ``$\idxscreen<\tarvecdim$''.

Let us first note that under hypothesis~\eqref{eq:WH}:
\begin{equation}
	\forall \idxtarelsorted\in \intervint{1}{\tarvecdim-1}:\ |\ktranspose{\dicomat}_{\backslash\tarvecdim}\spherec |_{\sortedentry{\idxtarelsorted}}=| \ktranspose{\dicomat}_{\backslash\tarvecdim}\spherec |_{\entry{\idxtarelsorted}},
\end{equation}
that is the $\idxtarelsorted$th largest element of $|\ktranspose{\dicomat}_{\backslash\tarvecdim}\spherec |$ is simply equal to its $\idxtarelsorted$th component.
The particularization of \eqref{eq: screening SLOPE all p} to $\idxscreen=\tarvecdim$ can then be rewritten as:
\begin{equation}\label{eq: screening test compressed form}
	\forall \idxps\in\intervint{1}{\tarvecdim}, \exists \idximp_\idxps\in \intervint{1}{\idxps}:\ \kvbar{
	\ktranspose{\atom}_{\column{\tarvecdim}}\spherec
	}
	< \thresinequ_{\idxps,\idximp_\idxps}
\end{equation}
where \(\thresinequ_{\idxps,\idximp}\) is defined \(\forall\idxps\in\intervint{1}{\pvdim}\) and \(\idximp\in\intervint{1}{\idxps}\) as
\begin{equation}
	\thresinequ_{\idxps,\idximp}
	\triangleq
	\screenthres_{\idxps,\idximp}
	-
	\sum_{\idxtarelsorted=\idximp}^{\idxps-1}
	\kvbar{
		\ktranspose{\dicomat}\spherec
	}_{\entry{\idxtarelsorted}}
	=
	\sum_{\idxtarelsorted=\idximp}^{\idxps-1} (\lambda\slopeweight_{\idxtarelsorted}
	-  \kvbar{
		\ktranspose{\dicomat}\spherec
	}_{\entry{\idxtarelsorted}}
	- \spherer)
	+ (\lambda \slopeweight_\idxps -\spherer)
	.
\end{equation}
We show hereafter that \eqref{eq: screening test compressed form} can be verified by only considering a ``well-chosen'' subset of thresholds
$\setinequ\subseteq\kset{\thresinequ_{\idxps,\idximp}}{\idxps\in\intervint{1}{\tarvecdim}, \idximp\in\intervint{1}{\idxps}}$, see \Cref{lemma: subset inequality} below.

If
\begin{equation}
	\idximp^\star(\idxps) 
	\triangleq
	\kargmax_{\idximp\in\intervint{1}{\idxps}} \thresinequ_{\idxps,\idximp}
	,
\end{equation}
we obviously have
\begin{equation}\label{eq: meaning r*}
	\kvbar{
	\ktranspose{\atom}_{\column{\tarvecdim}}\spherec
	}<\thresinequ_{\idxps,\idximp^\star(\idxps)}
	\iff
	\exists\idximp_{{\idxps}}\in\intervint{1}{\idxps}:\
	\kvbar{
	\ktranspose{\atom}_{\column{\tarvecdim}}\spherec
	}<\thresinequ_{\idxps,\idximp_{{\idxps}}}
	.
\end{equation}
In other words, for each $\idxps\in\intervint{1}{\tarvecdim}$, satisfying the inequality ``$\kvbar{\ktranspose{\atom}_{\column{\tarvecdim}}\spherec}<\thresinequ_{\idxps,\idximp}$'' for $\idximp=\idximp^\star(\idxps)$ is necessary and sufficient to ensure that it is verified for some $\idximp_{{\idxps}}\in\intervint{1}{\idxps}$.
Motivated by this observation, we show the following items below:
\textit{i)} $\idximp^\star(\idxps)$ can be evaluated $\forall \idxps\in\intervint{1}{\tarvecdim}$ with a complexity $\mathcal{O}(\tarvecdim)$; \textit{ii)} similarly to $\idximp$, only a subset of values of $\idxps\in\intervint{1}{\tarvecdim}$ are of interest to implement \eqref{eq: screening test compressed form}.

Let us define the function:
\begin{equation}\label{eq:def f}
	\kfuncdef{\screeningboundfunc}{\intervint{1}{\tarvecdim}}{\kR}[\idximp]
	[
		\sum_{\idxtarelsorted=\idximp}^{\tarvecdim} (\lambda\slopeweight_{\idxtarelsorted}
		- \kvbar{
			\ktranspose{\dicomat}\spherec}_{\entry{\idxtarelsorted}}
		- \spherer)
	]
	.
\end{equation}
We then have \(\forall \idxps\in\intervint{1}{\tarvecdim}\) and \(\idximp\in\intervint{1}{\idxps}\):
\begin{equation}\label{eq:def tau 2}
	\thresinequ_{\idxps,\idximp}
	= \screeningboundfunc(\idximp) - (\screeningboundfunc(\idxps) - \lambda \slopeweight_\idxps) -\spherer
	.
\end{equation}
In view of \eqref{eq:def tau 2}, the optimal value $\idximp^\star(\idxps)$ can be computed as
\begin{equation}
	\idximp^\star(\idxps) \label{eq:def idximpstar}
	=
	\kargmax_{\idximp\in\intervint{1}{\idxps}} \screeningboundfunc(\idximp).
\end{equation}
Considering \eqref{eq:def f}, we see that the evaluation of $\screeningboundfunc(\idximp)$ $\forall\idximp\in\intervint{1}{\tarvecdim}$ (and therefore $\idximp^\star(\idxps)$  $\forall \idxps\in\intervint{1}{\tarvecdim}$) can be done with a complexity scaling as $\mathcal{O}(\tarvecdim)$.
This proves item \textit{i)}.

Let us now show that only some specific indices $\idxps\in\intervint{1}{\tarvecdim}$ are of interest to implement \eqref{eq: screening test compressed form}.
Let
\begin{equation}
	\label{eq:def q^star(k)}
	\idxps^\star(\idxtarelsorted)
	\triangleq
	\kargmax_{\idxps\in\intervint{1}{\idxtarelsorted}} \screeningboundfunc(\idxps) - \lambda \slopeweight_\idxps
	,
\end{equation}
and define the sequence \(\{\idxps^{(\idxinequ)}\}_\idxinequ\) as
\begin{equation}\label{eq:def q set}
	\begin{cases}
		\idxps^{(1)}         & = \idxps^\star(\tarvecdim)                              \\
		\idxps^{(\idxinequ)} & = \idxps^\star(\idximp^\star(\idxps^{(\idxinequ-1)})-1)
	\end{cases}
\end{equation}
where the recursion is applied as long as $\idximp^\star(\idxps^{(\idxinequ-1)})>1$.\footnote{We note that the sequence $\{\idxps^{(\idxinequ)}\}_\idxinequ$ is strictly decreasing and thus contains at most $\tarvecdim$ elements.}
We then have the following result whose proof is available in \Cref{proof:lemma: subset inequality}:
\begin{lemma}\label{lemma: subset inequality}
	Let $\setinequ \triangleq \kset{\thresinequ_{\idxps,\idximp^{\star}(\idxps)}}{\idxps\in\{\idxps^{(\idxinequ)}\}_{\idxinequ}}$
	where $\{\idxps^{(\idxinequ)}\}_{\idxinequ}$ is defined in \eqref{eq:def q set}.
	Test \eqref{eq: screening test compressed form} is passed if and only if
	\begin{equation}\label{eq: threshold-based screening test}
		\forall \thresinequ\in\setinequ:\ |\ktranspose{\atom}_{\column{\tarvecdim}}\spherec|<\thresinequ
		.
	\end{equation}
\end{lemma}
\Cref{lemma: subset inequality} suggests the procedure described in Algorithm~\ref{alg:Implementation generalized SLOPE screening test} (with $\idxscreen=\tarvecdim$) to verify if \eqref{eq: screening test compressed form} is passed.
In a nutshell, the lemma states that only $\mathrm{card}(\setinequ)$ inequalities need to be taken into account to implement \eqref{eq: screening test compressed form}.
We note that $\mathrm{card}(\setinequ)\leq \tarvecdim$ since only one value of $\idximp$ (that is $\idximp^\star(\idxps)$) has to be considered for any $\idxps\in\intervint{1}{\tarvecdim}$.
This is in contrast with a brute-force evaluation of \eqref{eq: screening test compressed form} which requires the verification of $\mathcal{O}(\tarvecdim^2)$ inequalities.

\begin{algorithm}[t]
	\caption{
		\label{alg:Implementation generalized SLOPE screening test}
		Check if test \eqref{eq: screening SLOPE all p} is passed for $\idxscreen$ if it is passed for $\idxscreen'>\idxscreen$ 
	}
	\begin{algorithmic}[1]
		\REQUIRE{radius \(\spherer\geq0\)}, index $\idxscreen\in\intervint{1}{\tarvecdim}$,
		$\{\screeningboundfunc({\idximp})\}_{{\idximp}=1}^{\tarvecdim}$, $\{\idximp^\star({\idxps})\}_{{\idxps}=1}^{\tarvecdim}$,$\{\idxps^\star(\idxtarelsorted)\}_{\idxtarelsorted=1}^{\tarvecdim}$

		\STATE $\idxps = \idxps^\star(\idxscreen)$

		\STATE $\mathrm{test} = 1$
		\STATE $\mathrm{run} = 1$

		\vspace*{.2em}
		\WHILE{$\mathrm{run} == 1$}
		\STATE $\thresinequ = \screeningboundfunc(\idximp^\star(\idxps))-\screeningboundfunc(\idxps)+(\lambda\slopeweight_\idxps-\spherer)$ \COMMENT{Evaluation of current threshold, see \eqref{eq:def tau 2}}

		\IF{$|\ktranspose{\atom}_{\column{\idxscreen}} \spherec|\geq \thresinequ$}
		\STATE $\mathrm{test} = 0$ \COMMENT{Test failed}
		\STATE $\mathrm{run} = 0$        \COMMENT{Stops the recursion}
		\ENDIF

		\IF{$\idximp^\star(\idxps)>1$} 
		\STATE $\idxps = \idxps^\star(\idximp^\star(\idxps)-1)$ \COMMENT{Next value of $\idxps$ to test, see \eqref{eq:def q set}}
		\ELSE
		\STATE $\mathrm{run} = 0$        \COMMENT{Stops the recursion}
		\ENDIF

		\ENDWHILE

		\RETURN{$\mathrm{test}$ ($=1$ if test passed and $0$ otherwise)}
	\end{algorithmic}
\end{algorithm}

We finally emphasize that the procedure described in 
Algorithm~\ref{alg:Implementation generalized SLOPE screening test}
also applies to $\idxscreen<\tarvecdim$ as long as the screening test is passed for all $\idxscreen'>\idxscreen$.
More specifically, if test \eqref{eq: screening SLOPE all p} is passed for all $\idxscreen'\in\intervint{\idxscreen+1}{\tarvecdim}$, then its particularization to atom $\atom_{\column{\idxscreen}}$ reads
\begin{equation}\label{eq:concatenated test 3}
	\forall \thresinequ\in \setinequ':\ \kvbar{
	\ktranspose{\atom}_{\column{\idxscreen}}\spherec
	}
	< \thresinequ
\end{equation}
for some $\setinequ'\subseteq \setinequ$.

Indeed, if screening test \eqref{eq: screening SLOPE all p} is passed for all $\idxscreen'\in\intervint{\idxscreen+1}{\tarvecdim}$, the corresponding elements can be discarded from the dictionary and we obtain a reduced problem only involving atoms $\{\atom_{\column{\idxscreen'}}\}_{\idxscreen'\in\intervint{1}{\idxscreen}}$.
Since \eqref{eq:WH} is assumed to hold, $\atom_{\column{\idxscreen}}$ attains the smallest absolute inner product with $\spherec$ and we end up with the same setup as in the case ``$\idxscreen=\tarvecdim$''. 
In particular, if screening test \eqref{eq: screening SLOPE all p} is passed for all $\idxscreen'\in\intervint{\idxscreen+1}{\tarvecdim}$, \Cref{lemma: subset inequality} still holds for $\atom_{\column{\idxscreen}}$ by letting $\idxps^{(1)}=\idxps^\star(\idxscreen)$ in the definition of the sequence $\{\idxps^{(\idxinequ)}\}_\idxinequ$ in \eqref{eq:def q set}.

To conclude this section, let us summarize the complexity needed to implement \Cref{alg:Implementation generalized SLOPE screening test 2,alg:Implementation generalized SLOPE screening test}.
First, \Cref{alg:Implementation generalized SLOPE screening test 2} requires the entries $|\ktranspose{\dicomat}\spherec|$ to be sorted to satisfy hypothesis \eqref{lemma:nesting test}. This involves a complexity $\mathcal{O}(\tarvecdim\log\tarvecdim)$.
Moreover, the sequences $\{\screeningboundfunc({\idximp})\}_{{\idximp}=1}^{\tarvecdim}$, $\{\idximp^\star({\idxps})\}_{{\idxps}=1}^{\tarvecdim}$, $\{\idxps^\star(\idxtarelsorted)\}_{\idxtarelsorted=1}^{\tarvecdim}$ can be evaluated with a complexity $\mathcal{O}(\tarvecdim)$.
Finally, the main recursion in \Cref{alg:Implementation generalized SLOPE screening test 2} implies to run \Cref{alg:Implementation generalized SLOPE screening test} $\nscreen$ times, where $\nscreen$ is the number of atoms passing test \eqref{eq: screening SLOPE all p}.
Since \Cref{alg:Implementation generalized SLOPE screening test} requires to verify at most $\ninequ=\card[\setinequ]$ inequalities, the overall complexity of the main recursion scales as $\mathcal{O}(\nscreen\ninequ)$.
Overall, the complexity of \Cref{alg:Implementation generalized SLOPE screening test 2} is therefore $\mathcal{O}(\tarvecdim\log\tarvecdim+\nscreen\ninequ)$. \\


\section{Numerical simulations}
\label{sec:simus}
We present hereafter several simulation results demonstrating the effectiveness of the proposed screening procedure to accelerate the resolution of SLOPE.
This section is organized as follows.
In \Cref{subsec:simu:setup}, we present the experimental setups considered in our simulations.
In \Cref{subsec:simu:effectiveness} we compare the effectiveness of different screening strategies.
In \Cref{subsec:simu:bench}, we show that our methodology enables to reach better convergence properties for a given computational budget.\\


\subsection{Experimental setup} \label{subsec:simu:setup}
We detail below the experimental setups used in all our numerical experiments.

\vspace*{0.2cm}
\noindent
\textit{Dictionaries and observation vectors}:
New realizations of \(\dicomat\) and \(\obs\) are drawn for each trial as follows.
The observation vector is generated according to a uniform distribution on the \(\obsdim\)-dimensional sphere.
The elements of \(\dicomat\) obey one of the following models:\vspace*{0.1cm}
\begin{enumerate}
	\item the entries are i.i.d. realizations of a centered Gaussian,
	\item the entries are i.i.d. realizations of a uniform distribution on \([0,1]\),
	\item the columns are shifted versions of a Gaussian curve.\vspace*{0.1cm}
\end{enumerate}
For all distributions, the columns of \(\dicomat\) are normalized to have unit $\ell_2$-norm.
In the following, these three options will be respectively referred to as ``Gaussian'', ``Uniform'' and ``Toeplitz''.\vspace*{0.2cm}

\noindent
\textit{Regularization parameters}:
We consider three different choices for the sequence \(\{\slopeweight_\idxtarelsorted\}_{\idxtarelsorted=1}^\tarvecdim\), each of them corresponding to a different instance of the well-known OSCAR problem~\cite[Eq.~(3)]{Bondell2007}. More specifically, we let
\begin{equation}
	\label{eq:def-seq-oscar}
	\forall \idxtarelsorted\in\intervint{1}{\tarvecdim}:\
	\slopeoscar_\idxtarelsorted
	\triangleq
	\oscarparam_1 + \oscarparam_2 (\tarvecdim-\idxtarelsorted)
\end{equation}
where \(\oscarparam_1\), \(\oscarparam_2\) are nonnegative parameters chosen so that \(\slopeoscar_1=1\) and \(\slopeoscar_\tarvecdim\in\{.9, .1, 10^{-3}\}\).
In the sequel, these parametrizations will respectively be referred to as ``\oscar{1}'', ``\oscar{2}'' and ``\oscar{3}''.\\


\subsection{Performance of screening strategies}
\label{subsec:simu:effectiveness}

We first compare the effectiveness of different screening strategies described in \Cref{sec:screening-rule}.
More specifically, we evaluate the proportion of zero entries in \(\pvopt\) -- the solution of SLOPE problem~\eqref{eq:primal problem} -- that can be identified by tests \eqref{eq: safe screening for SLOPE p=q-1}, \eqref{eq: safe screening for SLOPE p=0} and \eqref{eq: screening SLOPE all p} as a function of the ``quality'' of the safe sphere.
These tests will respectively be referred to as ``\testA{}'', ``\testB{}'' and ``\testC{}'' in the following.
Figures \ref{fig:comparing choice of q_r} (see \Cref{sec:SLOPE screening main}) and \ref{fig:effectiveness} represent this criterion of performance as a function of some parameter \(\spherer_0\) (described below) and different values of the ratio \(\lambda / \lambda_{\max}\).
The results are averaged over \(50\) realizations.
For each simulation trial, we draw a new realization of \(\obs\in\kR^{\xpm}\) and \(\dicomat\in\kR^{\xpm\times\xpn}\) according to the distributions described in \Cref{subsec:simu:setup}. We consider Toeplitz dictionaries in \Cref{fig:comparing choice of q_r} and Gaussian dictionaries in \Cref{fig:effectiveness}.

The safe sphere used in the screening tests is constructed as follows. A primal-dual solution \((\pv_a,\dv_a)\) of problems~\eqref{eq:primal problem} and~\eqref{eq:dual problem} is evaluated with ``high-accuracy'', \textit{i.e.}, with a duality GAP of \(10^{-14}\) as stopping criterion.
More precisely, \(\pv_a\) is first evaluated by solving the SLOPE problem~\eqref{eq:primal problem} with the algorithm proposed in~\cite{Bogdan2015}.
{To evaluate \(\dv_a\), we extend the so-called ``dual scaling'' operator \cite[Section~3.3]{Ghaoui2010} to the SLOPE problem:
we let \(\dv_a = (\obs- \dicomat\pv_a)/\beta(\obs- \dicomat\pv_a)\)  where
\begin{equation}
	\label{eq:def dual scaling}
	\forall \bfz\in\kR^\obsdim:\ \beta(\bfz)
	\triangleq
	\max\kparen{
		1,
		\max_{\idxps\in\intervint{1}{\tarvecdim}}
		\frac{
			\sum_{\idxtarelsorted=1}^\idxps \kvbar{\ktranspose{\dicomat}\bfz}_{\sortedentry{\idxtarelsorted}}
		}{
			\lambda\sum_{\idxtarelsorted=1}^\idxps \slopeweight_\idxtarelsorted
		}
	}
	.
\end{equation}
The couple \((\pv_a,\dv_a)\) is then used to construct a sphere $\sphereregion(\spherec_a,\spherer_a)$ in \(\kR^\obsdim\) whose parameters are given by
\begin{subequations}
	\begin{align}
		\label{eq:xp-effectiveness:spherec}
		\spherec \;=\; & \dv_a               \\
		\label{eq:xp-effectiveness:spherer}
		\spherer 	\;=\; & \spherer_0 + \sqrt{
			2\kparen{
				\primfun(\pv_a) - \dualfun(\dv_a)
			}
		}
	\end{align}
\end{subequations}
where \(\spherer_0\) is a nonnegative scalar.}
We note that for \(\spherer_0=0\), the latter sphere corresponds to the GAP safe sphere described in~\eqref{eq:gap sphere}.\footnote{
	We note that the GAP safe sphere derived in \cite{Ndiaye2017} for problem~\eqref{eq:Ghaoui_GenericProblem} extends to SLOPE since 1) the dual problem has the same mathematical form and 2) its derivation does not leverage the definition of the dual feasible set.
}
Hence,~\eqref{eq:xp-effectiveness:spherec} and~\eqref{eq:xp-effectiveness:spherer} define a safe sphere for any choice of the nonnegative scalar \(\spherer_0 \geq 0\).

\begin{figure}
	\includegraphics[width=\columnwidth]{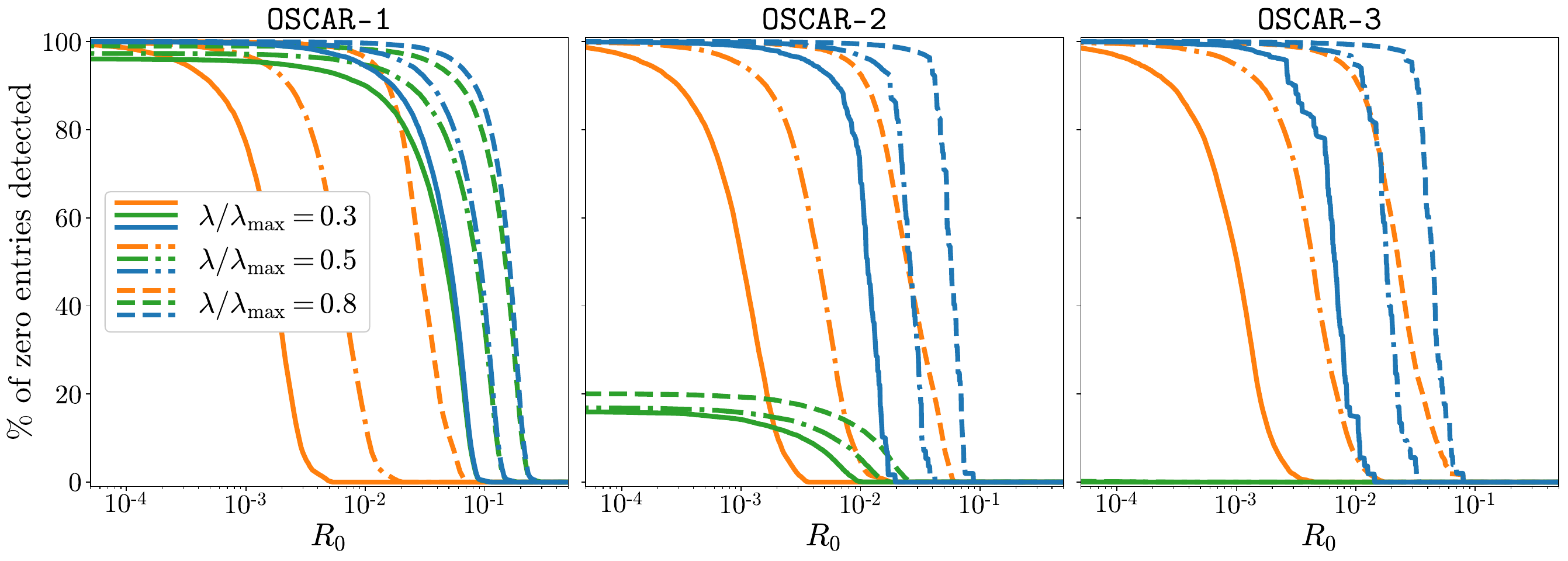}
	\caption{
	\label{fig:effectiveness}
	Percentage of zero entries in the solution of the SLOPE problem identified by \testA{} (orange lines), \testB{} (green lines) and \testC{} (blue lines) as a function of \(\spherer_0\) for the Gaussian dictionary, three values of \(\lambda/\lambda_{\max}\) and three parameter sequences \(\{\slopeweight_\idxtarelsorted\}_{\idxtarelsorted=1}^\tarvecdim\).
	}
\end{figure}

\Cref{fig:comparing choice of q_r} concentrates on the sequence \oscar{1} whereas  each subfigure corresponds to a different choice for \(\{\slopeweight_\idxtarelsorted\}_{\idxtarelsorted=1}^\tarvecdim\) in \Cref{fig:effectiveness}.
For the three considered screening strategies, we observe that the detection performance  decreases as \(\spherer_0\) increases.
Interestingly, different behaviors can be noticed.
For all simulation setups, \testA{} reaches a detection rate of \(100\%\) whenever \(\spherer_0\) is sufficiently small.
The performance of \testB{} varies from one sequence to another:
it outperforms \testA{} for \oscar{1}, is able to detect at most \(20\%\) of the zeros for \oscar{2} and fail for all values of \(\spherer_0\) for \oscar{3}.
Finally, \testC{} outperforms quite logically the two other strategies.
The gap in performance depends on both the considered setup and the radius \(\spherer_0\) but  can be quite significant in some cases.
For example, when \(\lambda/\lambda_{\max}=0.5\) and \(\spherer_0=10^{-2}\), there is \(80\%\) more entries passing \testC{} than \testA{} for all parameter sequences.


These results may be explained as follows.
First, we already mentioned in \Cref{sec:screening-rule} that when the radius of the safe sphere is sufficiently small (that is, when \(\spherer_0\) is close to zero), \testA{} is expected to be the best\footnote{in the sense defined in \cref{footnote:definition of best screening test} page~\pageref{footnote:definition of best screening test}.} screening test within the family of tests defined in \Cref{th: safe screening for SLOPE}.
Similarly, if the SLOPE weights satisfy \(\slopeweight_1=\slopeweight_\tarvecdim\), we showed in \Cref{lemma:optimality p=0 for LASSO} that no test in \Cref{th: safe screening for SLOPE} can outperform \testB{}.
Hence, one may reasonably expect that this conclusion remains valid whenever \(\slopeweight_1\simeq\slopeweight_\tarvecdim\), as observed for the sequence \oscar{1} in our simulations.
On the other hand, passing \testB{} becomes more difficult as parameter \(\slopeweight_\tarvecdim\) is small.
As a matter of fact, the test will \emph{never} pass when \(\slopeweight_\tarvecdim=0\).
In our experiments, the sequences \(\{\slopeweight_\idxtarelsorted\}_{\idxtarelsorted=1}^\tarvecdim\) are such that \(\slopeoscar_\tarvecdim\) is close to zero for \oscar{2} and \oscar{3}.
Finally, since \testC{} encompasses the two other tests, it is expected to always perform at least as well as the latter.\\


\subsection{Benchmarks}
\label{subsec:simu:bench}
As far as our simulation setup is concerned, the results presented in the previous section show a significant advantage in implementing \testC{} in terms of detection performance.
However, this conclusion does not include any consideration about the numerical complexity of the tests.
We note that, although the proposed screening rules can lead to a significant reduction of the problem dimensions, our tests also induce some additional computational burden.
In particular, we emphasized in \Cref{sec:implementation SLOPE screening} that \testC{} can be verified for all atoms of the dictionary with a complexity \(\calO(\tarvecdim\log\tarvecdim + \ninequ\nscreen)\) where \(\ninequ\leq \tarvecdim\) is a problem-dependent parameter and \(\nscreen\) is the number of atoms passing the test.
Moreover, we also note that, as far as a GAP safe sphere is considered in the implementation of the tests, its construction requires the identification of a dual feasible point $\dv$ and this operation typically induces a computational overhead of \(\mathcal{O}(\tarvecdim\log \tarvecdim)\) (see below for more details).

In this section, we therefore investigate the benefits (from a ``complexity-accuracy trade-off'' point of view) of interleaving the proposed safe screening methodology with the iterations of an accelerated proximal gradient algorithm~\cite{Bogdan2015}.
In all our tests, we consider the GAP safe sphere defined in \eqref{eq:gap sphere}.
The primal point used in the construction of the GAP sphere corresponds to the current iterate of the solving procedure, say $\pv^{(\idxit)}$. A dual feasible point $\dv^{(\idxit)}$ is constructed as
\begin{align}
	\dv^{(\idxit)} = \frac{\obs- \dicomat\pv^{(\idxit)}}{\beta(\obs- \dicomat\pv^{(\idxit)})}
\end{align}
where $\kfuncdef{\beta}{\kR^\obsdim}{\kR^\obsdim}$ is either defined as in \eqref{eq:def dual scaling} or as follows:
\begin{equation}
	\label{eq:def dual scaling 2}
	\forall \bfz\in\kR^\obsdim:\ \beta(\bfz)
	\triangleq
	\max
	\kparen{
		1,
		\displaystyle
		\max_{\idxtarelsorted\in\intervint{1}{\tarvecdim}}
		\frac{
			\adjustedvbar\ktranspose{\dicomat}\bfz\adjustedvbar_{\sortedentry{\idxtarelsorted}}
		}{
			\lambda \slopeweight_\idxtarelsorted
		}
	}.
\end{equation}
\eqref{eq:def dual scaling} matches the standard definition of the ``dual scaling'' operator proposed in \cite[Section~3.3]{Ghaoui2010} whereas \eqref{eq:def dual scaling 2} corresponds to the option considered in \cite{Bao:2020dq}.\footnote{
	See companion code of \cite{Bao:2020dq} available at\\ \url{https://github.com/brx18/Fast-OSCAR-and-OWL-Regression-via-Safe-Screening-Rules/tree/1e08d14c56bf4b6293899ae2092a5e0238d27bf6}.}
We notice that the two options require to sort the elements of $\kvbar{\ktranspose{\dicomat}\bfz}$ and thus lead to a complexity overhead scaling as $\mathcal{O}(\tarvecdim\log \tarvecdim)$.

In our simulations, we consider the four following solving strategies:\vspace{0.2cm}
\begin{enumerate}
	\item Run the proximal gradient procedure~\cite{Bogdan2015} with \textit{no} screening.
	\item Interleave some iterations of the proximal gradient algorithm with \testB{} and construct the dual feasible point with \eqref{eq:def dual scaling}.
	\item Interleave some iterations of the proximal gradient algorithm with \testB{} and construct the dual feasible point with \eqref{eq:def dual scaling 2}.
	\item Interleave some iterations of the proximal gradient algorithm with \testC{} and construct the dual feasible point with \eqref{eq:def dual scaling}.

	      \vspace{0.2cm}
\end{enumerate}
These strategies will respectively be denoted ``\algoSlope{}'', ``\algoSlopeScreening{}'', ``\algoSlopeBao{}'' and ``\algoSlopeScreeningp{}'' in the sequel.
We note that \algoSlopeBao{} closely matches the solving procedure considered in \cite{Bao:2020dq}.

We compare the performance of these solving strategies by resorting to Dolan-Moré profiles~\cite{Dolan2002}.
More precisely, we run each procedure for a given budget of time (that is
the algorithm is stopped after a predefined amount of time) on \(\nbRun=\xpnbrep\) different instances of the SLOPE problems.
In \algoSlopeScreening{}, \algoSlopeBao{} and \algoSlopeScreeningp{}, the screening procedure is applied once every \xpscreeningitparam{} iterations.
Each problem instance is generated by drawing a new dictionary \(\dicomat\in\kR^{\xpm\times\xpn}\) and observation vector \(\obs\in\kR^{\xpm}\) according to the distributions described in \Cref{subsec:simu:setup}.
We then compute the following performance profile for each solver \(\idSolver\in\kbrace{\text{\algoSlope{}, \algoSlopeScreening{}, \algoSlopeBao{}, \algoSlopeScreeningp{}}}\):
\begin{equation}
	\label{eq:exp:def_performance_profile}
	\rho_{\idSolver} (\xpgapvar) \triangleq 100\,\frac{\card[\kset{\idRun\in\intervint{1}{\nbRun}}{d_{\idRun,\idSolver}\leq \xpgapvar}]}{\nbRun}
	\quad
	\forall \xpgapvar\in\kR+
\end{equation}
where
\(d_{\idRun, \idSolver}\)  denotes the dual gap attained by solver \(\idSolver\) for problem instance \(\idRun\).
\(\rho_{\idSolver}(\xpgapvar)\) thus represents the (empirical) probability that solver \(\idSolver\) reaches a dual gap no greater than \(\xpgapvar\) for the considered budget of time.

\Cref{fig:bench} presents the performance profiles obtained for three types of dictionaries (Gaussian, Uniform and Toeplitz) and three different weighting sequences \(\{\slopeweight_\idxtarelsorted\}_{\idxtarelsorted=1}^\tarvecdim\) (\oscar{1}, \oscar{2} and \oscar{3}). The results are displayed for \(\lambda/\lambda_{\max} = 0.5\) but similar performance profiles have been obtained for other values of the ratio \(\lambda/\lambda_{\max}\).
All algorithms are implemented in Python with Cython bindings and experiments are run on a Dell laptop, 1.80  GHz, Intel Core i7.
For each setup, we adjusted the time budget so that \(\rho_{\text{\algoSlopeScreeningp{}}} (10^{-8})\simeq50\%\) for the sake of comparison.

As far as our simulation setup is concerned, these results show that the proposed screening methodologies improve the solving accuracy as compared to a standard proximal gradient.
\algoSlopeScreeningp{} improves the average accuracy over \algoSlope{} in all the considered settings. The gap in performance depends on the setup but is generally quite significant.
\algoSlopeScreening{} also enhances the average accuracy in most cases and performs at least comparably to \algoSlopeBao{} in all setups.
As expected, the behavior of \algoSlopeScreening{} and \algoSlopeBao{} is more sensitive to the choice of the weighting sequence \(\{\slopeweight_\idxtarelsorted\}_{\idxtarelsorted=1}^\tarvecdim\).
In particular, the screening performance of these strategies decreases when \(\slopeweight_\tarvecdim\simeq 0\) as emphasized in \Cref{subsec:simu:effectiveness}.
This results in no accuracy gain over \algoSlope{} for the sequence \oscar{3} as illustrated in \Cref{fig:bench}.
Nevertheless, we note that, even in absence of gain, \algoSlopeScreening{} and \algoSlopeBao{} do not seem to significantly degrade the performance as compared to \algoSlope{}.
\\

\begin{figure}
	\includegraphics[width=\columnwidth]{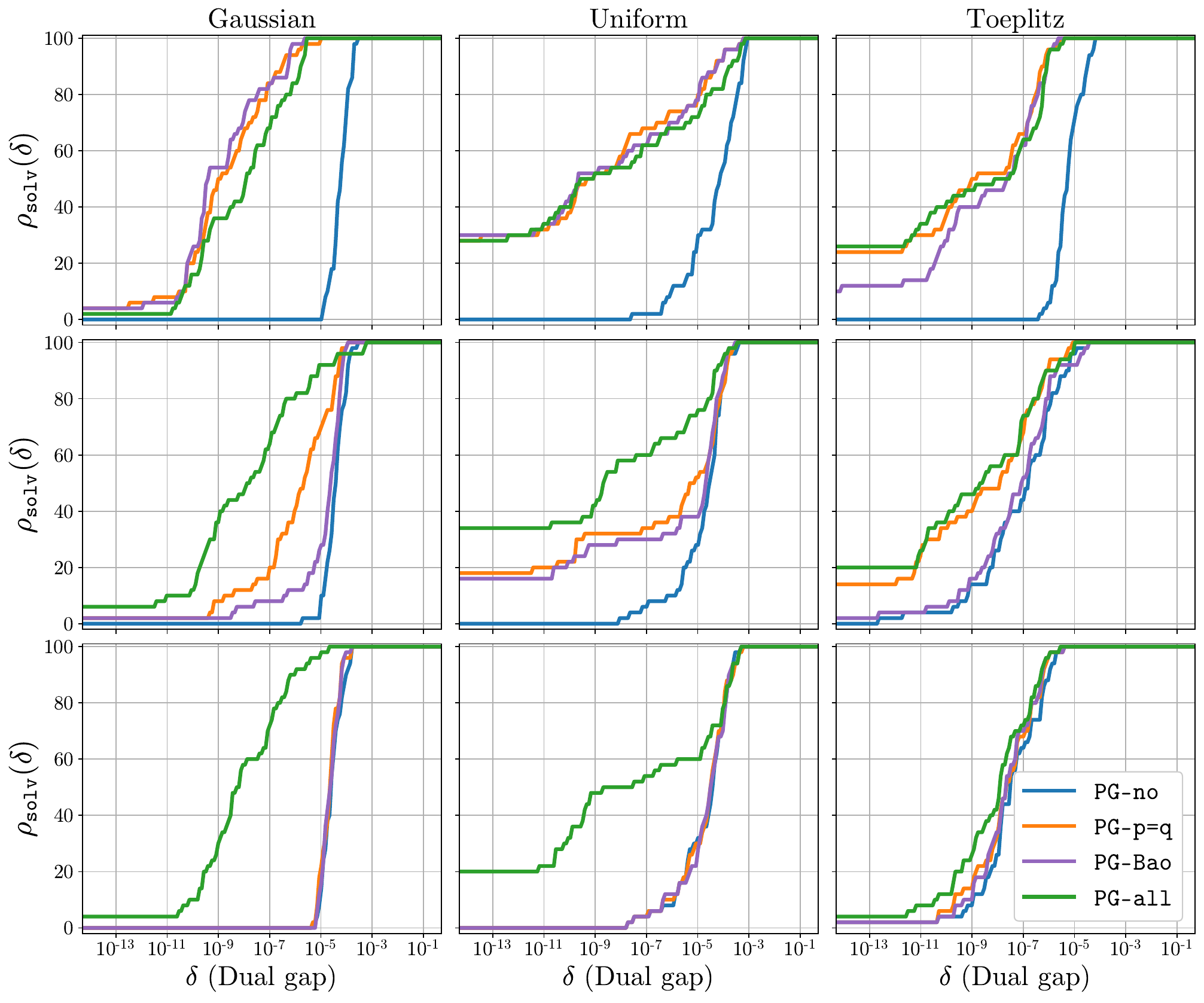}
	\caption{
		\label{fig:bench}
		Performance profiles of \algoSlope{}, \algoSlopeScreening{}, \algoSlopeBao{} and \algoSlopeScreeningp{} obtained for the ``Gaussian'' (column 1), ``Uniform'' (column 2) and ``Toeplitz'' (column 3) dictionaries and \(\lambda / \lambda_{\max}=0.5\) with a budget of time.
		First row: \oscar{1}, second row: \oscar{2} and third row: \oscar{3}.
	}
\end{figure}


\section{Conclusions}
\label{sec:conclusions}
In this paper we proposed a new methodology to safely identify the zeros of the solutions of the SLOPE problem.
In particular, we introduced a family of screening rules indexed by some parameters $\{\idximp_\idxps\}_{\idxps=1}^\tarvecdim$ where $\tarvecdim$ is the dimension of the primal variable.
Each test of this family takes the form of a series of \(\pvdim\) inequalities which, when verified, imply the nullity of some coefficient of the minimizers.
Interestingly, the proposed tests encompass standard ``sphere'' screening rule for LASSO as a particular case for some $\{\idximp_\idxps\}_{\idxps=1}^\tarvecdim$, although this choice does not correspond to the most effective test in the general case.
We then introduced an efficient numerical procedure to jointly  evaluate all the tests in the proposed family.
Our algorithm has a complexity $\mathcal{O}(\tarvecdim \log \tarvecdim + \ninequ\nscreen)$ where $\ninequ\leq \tarvecdim$ is some problem-dependent constant and $\nscreen$ is the number of elements passing at least one test of the family.
We finally assessed the performance of our screening strategy through numerical simulations and showed that the proposed methodology leads to significant improvements of the solving accuracy for a prescribed computational budget.  \\

\section*{Acknowledgments}
The authors would like to thank the anonymous reviewers for their thoughtful comments and for pointing out one technical flaw in the first version of the manuscript.\\

\appendix


\section{Miscellaneous results}\label{app:Miscellaneous results}
\Cref{subsec:app:sub-diff} reminds some useful results from convex analysis applied to the SLOPE problem~\eqref{eq:primal problem}.
\Cref{subsec:app:proof cns solutions slope is zero} provides a proof of~\eqref{eq:cns solution slope is zero}. In all the statements below, \(\partial\regslope(\pv)\) denotes the subdifferential of \reffunctext{\regslope} evaluated at \(\pv\).\\

\subsection{Some results of convex analysis}  \label{subsec:app:sub-diff}
We remind below several results of convex analysis that will be used in our subsequent derivations.
The first lemma provides a necessary and sufficient condition for \(\pvopt\in\kR^\pvdim\) to be a minimizer of the SLOPE problem~\eqref{eq:primal problem}:

\noindent
\begin{lemma}
	\label{lemma:fermat s rule}
	\(\text{\(\pvopt\) is a minimizer of~\eqref{eq:primal problem}}
	\ \Longleftrightarrow\
	\kinv{\lambda}\ktranspose{\dicomat}(\obs - \dicomat\pvopt)
	\in
	\partial\regslope(\pvopt)\).\\
\end{lemma}

\noindent
\Cref{lemma:fermat s rule} follows from a direct application of Fermat's rule~\cite[Proposition~16.4]{Bauschke2017} to problem~\eqref{eq:primal problem}.
We note that under condition~\eqref{eq:hyp slopeweigths}, \reffunctext{\regslope}
defines a norm on \(\kR^\tarvecdim\), see \eg \cite[Proposition~1.1]{Bogdan2013statistical} or~\cite[Lemma~2]{Zeng:2014AtomicNorm}.
The subdifferential \(\partial \regslope(\tarvec)\) is therefore well defined for all \(\tarvec\in\kR^\tarvecdim\) and writes as
\begin{equation}
	\label{eq:app:proof:subdifferential norm}
	\partial\regslope(\tarvec)
	=
	\kset{
		\subdiffvec\in\kR^\tarvecdim
	}{
		\ktranspose{\subdiffvec}\tarvec = \regslope(\tarvec)
		\text{ and }
		\dualregslope(\subdiffvec) \leq 1
	},
\end{equation}
where
\begin{equation}
	\dualregslope(\subdiffvec) \triangleq \sup_{\tarvec\in\kR^\pvdim}
	\ktranspose{\subdiffvec}\pv
	\ \text{ s.t.}\
	\regslope(\tarvec)\leq 1
\end{equation}
is the dual norm of
\reffunctext{\regslope}, see \eg \cite[Eq.~(1.4)]{bach2011book}.

The next lemma states a technical result which will be useful in the proof of \Cref{th:safe-screening} in \Cref{sec:app:main-proofs}:\\

\begin{lemma}
	\label{lemma:sudiff property}
	If \(\subdiffvec\in\partial\regslope(\pv)\), then \(\ktranspose{\pv}(\subdiffvec-\subdiffvec')
	\geq 0
	\,
	\forall\subdiffvec'\in\kR^\pvdim \text{ s.t. } \dualregslope(\subdiffvec')\leq 1.\)
\end{lemma}

\noindent
%
\begin{proof}
	Let \(\subdiffvec\in\partial\regslope(\pv)\).
	One has
	\begin{align}
		\subdiffvec\in\partial\regslope(\pv)
		\;\Longleftrightarrow\; &
		\pv\in\partial \regslope^*(\subdiffvec)
		\nonumber                 \\
		\;\Longleftrightarrow\; &
		\forall \subdiffvec'\in\kR^\pvdim, \;
		\regslope^*(\subdiffvec') \geq \regslope^*(\subdiffvec) + \kangle{\pv, \subdiffvec' - \subdiffvec}
	\end{align}
	where \reffunctext{\regslope^*} refers to the Fenchel conjugate of \reffunctext{\regslope}.
	The first equivalence is a consequence of \cite[Theorem~16.29]{Bauschke2017} and
	the second of the definition of the subdifferential set.
	\Cref{lemma:sudiff property} follows by noticing that \(\regslope^*(\subdiffvec')=0\) \(\forall\subdiffvec'\in\kR^\pvdim\) such that \(\dualregslope(\subdiffvec')\leq1\) by property of \reffunctext{\regslope^*} \cite[Item~(v) of Example 13.3]{Bauschke2017}. 
	%
	%
\end{proof}

\vspace{2em}
In the last lemma of this section, we provide a closed-form expression of the subdifferential and the dual norm of \reffunctext{\regslope}:\footnote{We note that an expression of the subdifferential of
	\reffunctext{\regslope}
	has already been derived in~\cite[Fact A.2 in supplementary material]{Bu2019Algo}. However, the expression of the subdifferential proposed in \Cref{lemma:sub-diff} has a more compact form and is better suited to our subsequent derivations.}

\begin{lemma}
	\label{lemma:sub-diff}
	The dual norm and the subdifferential of \(\regslope(\pv)\) respectively write:
	\begin{equation*}
		\begin{split}
			\nonumber
			\dualregslope(\subdiffvec)
			&=
			\max_{\idxps\in\intervint{1}{\pvdim}}
			\
			\frac{1}{\sum_{\idxtarelsorted=1}^{\idxps} \slopeweight_\idxtarelsorted}
			\sum_{\idxtarelsorted=1}^{\idxps} |\subdiffvec|_{\sortedentry{\idxtarelsorted}}
			,\\
			\nonumber
			\partial \regslope(\tarvec)
			&=
			\kset{
			\subdiffvec\in\kR^\tarvecdim
			}{
			\ktranspose{\subdiffvec}\tarvec = \regslope(\tarvec)
			\text{ and }
			\forall\idxps\in\intervint{1}{\tarvecdim}:\
			\sum_{\idxtarelsorted=1}^{\idxps} |\subdiffvec|_{\sortedentry{\idxtarelsorted}} \leq \sum_{\idxtarelsorted=1}^{\idxps} \slopeweight_\idxtarelsorted
			}.
		\end{split}
	\end{equation*}
\end{lemma}

\begin{proof}
	The expression of the dual norm is a direct consequence of~\cite[Lemma~4]{Zeng:2014AtomicNorm}.
	More precisely, the authors showed that
	\begin{equation}
		\label{eq:app:proof:dual norm atomic norm}
		\dualregslope(\subdiffvec) = \max_{\bfv\in\bigcup_{\idxps=1}^\tarvecdim\atomicSet_\idxps} \; \ktranspose{\subdiffvec}\bfv
	\end{equation}
	where 
	\(\atomicSet_\idxps \triangleq
	\kset{
		\tfrac{1}{
			\sum_{\idxtarelsorted=1}^{\idxps}
			\slopeweight_{\idxtarelsorted}
		}
		\atomicsignvec
	}{
		\atomicsignvec\in\{0,-1,+1\}^\tarvecdim,
		\card[\{\idxtarel : \atomicsignvec_{\entry{\idxtarel}} \neq 0\}] = \idxps
	}\) for all \(\idxps\in\intervint{1}{\pvdim}\).
	The expression of \reffunctext{\dualregslope}
	given in \Cref{lemma:sub-diff} is a compact rewriting of 
	\eqref{eq:app:proof:dual norm atomic norm}
	that can be obtained as follows.
	See first that for all \(\idxps\in\intervint{1}{\pvdim}\),
	\begin{equation} \label{eq:upper bound proof sub diff}
		\max_{\bfv\in\atomicSet_\idxps} \; \ktranspose{\subdiffvec}\bfv
		\leq
		\frac{1}{
			\sum_{\idxtarelsorted=1}^{\idxps}
			\slopeweight_{\idxtarelsorted}
		}
		\sum_{\idxtarelsorted=1}^\idxps
		\kvbar{\subdiffvec}_{\sortedentry{\idxtarelsorted}}
		.
	\end{equation}
	Second, for \(\idxps\in\intervint{1}{\pvdim}\), let \(\mathcal{J}_\idxps\subset\intervint{1}{\pvdim}\) be a set \(\idxps\) distinct indices such that \(|\subdiffvec_{\entry{\idxtarel}}| \geq \kvbar{\subdiffvec}_{\sortedentry{\idxps}}\) for all \(\idxtarel\in\mathcal{J}_\idxps\).
	Then, the upper bound in~\eqref{eq:upper bound proof sub diff} is attained by evaluating the left-hand side at \(\bfv\in\atomicSet_\idxps\) defined as
	\begin{equation}
		\forall \idxtarel\in\intervint{1}{\pvdim}:\quad
		\bfv_{\entry{\idxtarel}}=
		\begin{cases}
			\tfrac{1}{
				\sum_{\idxtarelsorted=1}^{\idxps}
				\slopeweight_{\idxtarelsorted}
			}\,\sign[\subdiffvec_{\entry{\idxtarel}}] & \text{ if } \idxtarel\in\mathcal{J}_\idxps \\
			0                                         & \text{otherwise.}
		\end{cases}
	\end{equation}
	The expression of the subdifferential follows from \eqref{eq:app:proof:subdifferential norm} by plugging the expression of the dual norm in the inequality ``\(\dualregslope(\subdiffvec)\leq1\)''.
\end{proof}

\vspace{0.2cm}

\subsection{Proof of~\texorpdfstring{\eqref{eq:cns solution slope is zero}}{(\ref*{eq:cns solution slope is zero})}} \label{subsec:app:proof cns solutions slope is zero}

We first observe that
\begin{equation}
	\label{eq:applying-fermat-rule}
	\text{\(\0_\tarvecdim\) is not a minimizer of~\eqref{eq:primal problem}}
	\Longleftrightarrow
	\kinv{\lambda} \ktranspose{\dicomat} \obs
	\notin
	\partial \regslope(\0_\tarvecdim)
	,
\end{equation}
as a direct consequence of \Cref{lemma:fermat s rule}.
Particularizing the expression of \(\partial\regslope(\tarvec)\) in \Cref{lemma:sub-diff} to \(\tarvec={\bf0}_\tarvecdim\), the right-hand side of~\eqref{eq:applying-fermat-rule} can equivalently be rewritten as
\begin{equation}
	\label{eq:proof-0issol-buf}
	\exists\idxps\in\intervint{1}{\tarvecdim}:\;
	\kinv{\lambda} \sum_{\idxtarelsorted=1}^{\idxps} \adjustedvbar{\ktranspose{\dicomat}\obs}\adjustedvbar_{\sortedentry{\idxtarelsorted}} >
	\sum_{\idxtarelsorted=1}^{\idxps}\slopeweight_\idxtarelsorted
	.
\end{equation}
Since \(\slopeweight_1>0\) and the sequence \(\{\slopeweight_\idxtarelsorted\}_{\idxtarelsorted=1}^{\tarvecdim}\) is nonnegative by hypothesis~\eqref{eq:hyp slopeweigths},~\eqref{eq:proof-0issol-buf} can also be rewritten as
\begin{equation}
	\label{eq:proof-0issol-buf 2}
	\exists\idxps\in\intervint{1}{\tarvecdim}:\;
	\lambda
	<
	\frac{
		\sum_{\idxtarelsorted=1}^{\idxps} \adjustedvbar{\ktranspose{\dicomat}\obs}\adjustedvbar_{\sortedentry{\idxtarelsorted}}
	}{
		\sum_{\idxtarelsorted=1}^{\idxps}\slopeweight_\idxtarelsorted
	}
	.
\end{equation}
The statement in~\eqref{eq:cns solution slope is zero} then follows by noticing that the right-hand side of \eqref{eq:lambda-such-that-0-is-sol} is a compact reformulation of~\eqref{eq:proof-0issol-buf 2}.\\


\section{Proofs related to screening tests}
\label{sec:app:main-proofs}

\subsection{Proof of \texorpdfstring{\Cref{th:safe-screening}}{Theorem~\ref*{th:safe-screening}}} \label{subsec:app:proof ideal screening Slope}
In this section, we provide the technical details leading to \eqref{eq: ideal safe screening test}.
Our derivation leverages the Fermat's rule and the expression of the subdifferential derived in \Cref{lemma:sub-diff}.

We prove~\eqref{eq: ideal safe screening test} by contraposition.
More precisely, we show that if \(\pvopt_{\entry{\idxscreen}}\neq0\) for some \(\idxscreen\in\intervint{1}{\tarvecdim}\), then 
\begin{equation}
	\label{proof:eq:target statement}
	\exists \idxps_0 \in\intervint{1}{\pvdim},\; \adjustedvbar\ktranspose{\atom}_{\column{\idxscreen}}\dvopt\adjustedvbar + \sum_{\idxtarelsorted=1}^{\idxps_0-1} \adjustedvbar\ktranspose{\dicomat}_{\backslash \idxscreen}\dvopt\adjustedvbar_{\sortedentry{\idxtarelsorted}}
	=
	\lambda \sum_{\idxtarelsorted=1}^{\idxps_0} \slopeweight_\idxtarelsorted
	.
\end{equation}
Using \Cref{lemma:fermat s rule} and the following connection between primal-dual solutions (see~\cite[Section~2.5]{Bogdan2013statistical})
\begin{equation}
	\label{eq:proof:optimality condition dual and primal sol}
	\dvopt = \obs - \dicomat \pvopt
	,
\end{equation}
we have that \(\pvopt\) is a minimizer of~\eqref{eq:primal problem} if and only if
\begin{equation}
	\label{eq:fermatsrule}
	\subdiffvec^\star \triangleq
	\kinv{\lambda}\ktranspose{\dicomat}\dvopt
	\in
	\partial\regslope(\pvopt)
	.
\end{equation}
In the rest of the proof, we will use \Cref{lemma:sudiff property} with \(\tarvec=\pvopt\), \(\subdiffvec=\subdiffvec^\star\) and different instances of vector \(\subdiffvec'\) to prove our statement.
First, let us define \(\subdiffvec'\in\kR^\tarvecdim\) as
\begin{equation*}
	\begin{split}
		\subdiffvec'_{\entry{\idxtarel}}
		\;=\;&
		\subdiffvec^\star_{\entry{\idxtarel}}
		\quad
		\forall\idxtarel\in\intervint{1}{\pvdim}\setminus\{\idxscreen\}, \\
		%
		\subdiffvec'_{\entry{\idxscreen}}
		\;=\;&
		0.
	\end{split}
\end{equation*}
It is easy to verify that \(\dualregslope(\subdiffvec')\leq 1\).
Applying \Cref{lemma:sudiff property} then leads to
\begin{equation}
	\subdiffvec^\star_{\entry{\idxscreen}}\pvopt_{\entry{\idxscreen}}\geq 0.
\end{equation}
Since \(\pvopt_{\entry{\idxscreen}}\) is assumed to be nonzero, we then have
\begin{equation}
	\label{eq:complementary slackness}
	\mathrm{sign}\big(\subdiffvec^\star_{\entry{\idxscreen}}\big)\,\mathrm{sign}\big(\pvopt_{\entry{\idxscreen}}\big)\geq 0,
\end{equation}
where the equality holds if and only if \(\subdiffvec^\star_{\entry{\idxscreen}}=0\).

Second, let us consider the following choice for \(\subdiffvec'\in\kR^\tarvecdim\):
\begin{equation}\label{eq:def g' 2}
	\begin{split}
		\subdiffvec'_{\entry{\idxtarel}}
		\;=\;&
		\subdiffvec^\star_{\entry{\idxtarel}}
		\quad
		\forall\idxtarel\in\intervint{1}{\pvdim}\setminus\{\idxscreen\}, \\
		%
		\subdiffvec'_{\entry{\idxscreen}}
		\;=\;&
		\subdiffvec^{\star}_{\entry{\idxscreen}}+s\delta,
	\end{split}
\end{equation}
where
\begin{equation}
	\label{eq:def s}
	s \triangleq
	\begin{cases}
		\mathrm{sign}\big(\subdiffvec^\star_{\entry{\idxscreen}}\big) & \text{ if } \subdiffvec^\star_{\entry{\idxscreen}} \neq 0 \\
		\mathrm{sign}\big(\pv^\star_{\entry{\idxscreen}}\big)
		                                                              & \text{ otherwise,}
	\end{cases}
\end{equation}
and \(\delta\) is any nonnegative scalar such that
\begin{equation} \label{eq:g' constraint 1}
	\dualregslope(\subdiffvec') \leq  1
	.
	%
\end{equation}
On the one hand, we note that \eqref{eq:g' constraint 1} is verified for $\delta=0$.
On the other hand, it can be seen that \eqref{eq:g' constraint 1} is violated as soon as $\delta>0$ by using the following arguments.
%
First, applying \Cref{lemma:sudiff property} with $\subdiffvec'$ defined as in \eqref{eq:def g' 2} leads to
\begin{equation}
	\label{eq:proof:second application lemma 2}
	-
	s
	\pvopt_{\entry{\idxscreen}}
	\delta
	\geq
	0.
\end{equation}
Second, using \eqref{eq:complementary slackness} and the definition of \(s\) in \eqref{eq:def s}, we must have \(s\pvopt_{\entry{\idxscreen}}>0\).
Hence, satisfying inequality \eqref{eq:g' constraint 1} necessarily implies that \(\delta= 0\).
The contraposition of this result implies:
\begin{equation}\label{eq: strict inequality exists}
	\forall \delta> 0,
	\exists\idxps_0\in\intervint{1}{\pvdim}:\;
	\sum_{\idxtarelsorted=1}^{\idxps_0}|\subdiffvec^\star|_{\sortedentry{\idxtarelsorted}}
	+ \delta
	>
	\sum_{\idxtarelsorted=1}^{\idxps_0} \slopeweight_\idxtarelsorted
\end{equation}
or equivalently
\begin{equation}
	\exists\idxps_0\in\intervint{1}{\pvdim}:\;
	\sum_{\idxtarelsorted=1}^{\idxps_0}|\subdiffvec^\star|_{\sortedentry{\idxtarelsorted}}
	=
	\sum_{\idxtarelsorted=1}^{\idxps_0} \slopeweight_\idxtarelsorted
	.
\end{equation}

Let us next emphasize that the range of values for $\idxps_0$ can be restricted
by choosing some suitable value for $\delta$.
	In particular, define \(\idxps_0'\in\intervint{1}{\pvdim}\) as
	\begin{align} \label{eq:def q0'}
		\idxps_0'
		\triangleq
		\min
		\kset{
			\idxps\in\intervint{1}{\pvdim}
		}{
			|\subdiffvec^\star_{\entry{\idxscreen}}| = |\subdiffvec^\star|_{\sortedentry{\idxps}}
		}
	\end{align}
and let
\begin{equation} \label{eq:condition on delta 2}
	0
	< \delta
	< |\subdiffvec^\star|_{\sortedentry{\idxps_0'-1}}-|\subdiffvec^\star|_{\sortedentry{\idxps_0'}}
\end{equation}
with the convention \(\subdiffvec^\star_{\sortedentry{0}}=+\infty\).
Considering $\subdiffvec'$ as defined in \eqref{eq:def g' 2} with $\delta$ satisfying~\eqref{eq:condition on delta 2}, we have that the first $\idxps_0'-1$ largest absolute elements of $\subdiffvec'$ and $\subdiffvec^\star$ are the same. Since $\dualregslope(\subdiffvec^\star)\leq1$, the inequality in the right-hand side of \eqref{eq: strict inequality exists} can therefore not be verified for $\idxps_0\in\intervint{1}{\idxps_0'-1}$. Hence, considering $\delta$ as in \eqref{eq:condition on delta 2}, we have
\begin{equation}
	\exists\idxps_0\in\intervint{\idxps_0'}{\pvdim}:\;
	\sum_{\idxtarelsorted=1}^{\idxps_0}|\subdiffvec^\star|_{\sortedentry{\idxtarelsorted}}
	=
	\sum_{\idxtarelsorted=1}^{\idxps_0} \slopeweight_\idxtarelsorted
	.
\end{equation}

We finally obtain our original assertion \eqref{proof:eq:target statement} by using the definition of \(\subdiffvec^\star\) in \eqref{eq:fermatsrule} and the fact that
\begin{equation}
	\sum_{\idxtarelsorted=1}^{\idxps_0}\adjustedvbar\ktranspose{\dicomat}\dvopt\adjustedvbar_{\sortedentry{\idxtarelsorted}}
	=
	\adjustedvbar\ktranspose{\atom}_{\column{\idxscreen}} \dvopt\adjustedvbar + \sum_{\idxtarelsorted=1}^{\idxps_0-1}\adjustedvbar\ktranspose{\dicomat}_{\backslash\idxscreen}\dvopt\adjustedvbar_{\sortedentry{\idxtarelsorted}}
\end{equation}
since $|\ktranspose{\atom}_{\column{\idxscreen}} \dvopt|=|\ktranspose{\dicomat} \dvopt|_{\sortedentry{\idxps_0'}}$ by definition of $\idxps_0'$ in \eqref{eq:def q0'} and $|\ktranspose{\dicomat} \dvopt|_{\sortedentry{\idxps_0'}}\geq |\ktranspose{\dicomat}\dvopt|_{\sortedentry{\idxps_0}}$ by definition of $\idxps_0\geq\idxps_0'$.\\

\subsection[Proof of Lemma~\ref{lemma:upper bound}]{Proof of \Cref{lemma:upper bound}}\label{sec:Proof of lemma:upper bound}

We first state and prove the following technical lemma:

\begin{lemma}
	\label{lemma:ordered_inequality}
	Let \(\bfg\in\kR^\tarvecdim\) and \(\bfh\in\kR^\tarvecdim\) be such that \(\bfg_{\entry{\idxtarel}}\leq \bfh_{\entry{\idxtarel}}\)\(\,\forall \idxtarel\in\intervint{1}{\pvdim}\).
	Then
	\begin{equation}
		\mbox{\(\bfg_{\sortedentry{\idxtarelsorted}}\leq \bfh_{\sortedentry{\idxtarelsorted}}\) \(\,\forall \idxtarelsorted\in\intervint{1}{\pvdim} \).}
	\end{equation}
\end{lemma}
\begin{proof}
	Let \(\idxtarelsorted\in\intervint{1}{\pvdim}\).
	We have by definition
	\begin{equation*}
		\begin{split}
			\bfh_{\sortedentry{\idxtarelsorted}}
			&=
			\max_{
				\substack{
					\calJ \subseteq \intervint{1}{\tarvecdim} \\
					\mathrm{card}(\calJ)=\idxtarelsorted}
			}
			\min_{\idxtarel\in\calJ} \bfh_{\entry{\idxtarel}},
			\nonumber \\
			&\geq \max_{
				\substack{
					\calJ \subseteq \intervint{1}{\tarvecdim} \\
					\mathrm{card}(\calJ)=\idxtarelsorted}
			}
			\min_{\idxtarel\in\calJ} \bfg_{\entry{\idxtarel}},\nonumber \\
			&= \bfg_{\sortedentry{\idxtarelsorted}}
			,
		\end{split}
	\end{equation*}
	where the inequality follows from our assumption \(\bfg_{\entry{\idxtarel}}\leq \bfh_{\entry{\idxtarel}}\) \(\forall \idxtarel\in\intervint{1}{\pvdim}\).
\end{proof}

\vspace{1em}
We are now ready to prove \Cref{lemma:upper bound}.
For any \(\idximp\in\intervint{1}{\idxps}\), we can write:
\begin{equation}
	\adjustedvbar\ktranspose{\atom}_{\column{\idxscreen}}\dvopt\adjustedvbar + \sum_{\idxtarelsorted=1}^{\idxps-1} \adjustedvbar\ktranspose{\dicomat}_{\backslash \idxscreen}\dvopt\adjustedvbar_{\sortedentry{\idxtarelsorted}}
	= \adjustedvbar\ktranspose{\atom}_{\column{\idxscreen}}\dvopt\adjustedvbar
	+ \sum_{\idxtarelsorted=1}^{\idximp-1} \adjustedvbar\ktranspose{\dicomat}_{\backslash \idxscreen}\dvopt\adjustedvbar_{\sortedentry{\idxtarelsorted}}
	+ \sum_{\idxtarelsorted=\idximp}^{\idxps-1} \adjustedvbar\ktranspose{\dicomat}_{\backslash \idxscreen}\dvopt\adjustedvbar_{\sortedentry{\idxtarelsorted}}
	.
\end{equation}
First, since \(\dvopt\) is dual feasible, we have:
\begin{equation}
	\label{eq:proof:upper bound 1}
	\sum_{\idxtarelsorted=1}^{\idximp-1}
	\adjustedvbar\ktranspose{\dicomat}_{\setminus\idxscreen} \dvopt\adjustedvbar_{\sortedentry{\idxtarelsorted}}
	\leq
	\lambda
	\sum_{\idxtarelsorted=1}^{\idximp-1}
	\slopeweight_{\idxtarelsorted}
	.
\end{equation}
We next show that if \(\dvopt\in\sphereregion(\spherec,\spherer)\), then
\begin{equation}
	\label{eq:proof:upper bound 2}
	\adjustedvbar\ktranspose{\atom}_{\column{\idxscreen}}\dvopt\adjustedvbar
	+ \sum_{\idxtarelsorted=\idximp}^{\idxps-1} \adjustedvbar\ktranspose{\dicomat}_{\backslash \idxscreen}\dvopt\adjustedvbar_{\sortedentry{\idxtarelsorted}} \leq
	\adjustedvbar{
	\ktranspose{\atom}_{\column{\idxscreen}}\spherec
	}\adjustedvbar
	+ \sum_{\idxtarelsorted=\idximp}^{\idxps-1} \adjustedvbar{
		\ktranspose{\dicomat}_{\backslash \idxscreen}\spherec
	}\adjustedvbar_{\sortedentry{\idxtarelsorted}}
	+
	(\idxps-\idximp+1)
	\spherer
	.
\end{equation}
We then obtain the result stated in the lemma by combining \eqref{eq:proof:upper bound 1}-\eqref{eq:proof:upper bound 2}.

Inequality \eqref{eq:proof:upper bound 2} can be shown as follows.
First,
\begin{equation}
	\forall \idxtarel\in\intervint{1}{\tarvecdim}: \max_{\dv\in\sphereregion(\spherec,\spherer)}|\ktranspose{\atom}_{\column{\idxtarel}}\dv| = |\ktranspose{\atom}_{\column{\idxtarel}}\spherec|+ \spherer.
\end{equation}
Hence,
\begin{equation}
	\kparen{\max_{\dv\in\sphereregion(\spherec,\spherer)}\adjustedvbar\ktranspose{\dicomat}_{\backslash \ell}\dv\adjustedvbar}_{\sortedentry{\idxtarelsorted}} = \adjustedvbar\ktranspose{\dicomat}_{\backslash \ell}\spherec\adjustedvbar_{\sortedentry{\idxtarelsorted}}+ \spherer
\end{equation}
where the maximum is taken component-wise in the left-hand side of the equation.
Applying \Cref{lemma:ordered_inequality}
with \(\bfg = |\ktranspose{\dicomat}_{\backslash \ell}\dv|\) and \(\bfh = \max_{\tilde{\dv}\in\sphereregion(\spherec,\spherer)}|\ktranspose{\dicomat}_{\backslash \ell}\tilde{\dv}|\), we have
\begin{equation}
	\label{eq:inequality max applying lemma}
	\forall \dv\in\sphereregion(\spherec,\spherer):\
	\adjustedvbar\ktranspose{\dicomat}_{\backslash \ell}\dv\adjustedvbar_{\sortedentry{\idxtarelsorted}}
	\leq
	\kparen{\max_{\tilde{\dv}\in\sphereregion(\spherec,\spherer)}\adjustedvbar\ktranspose{\dicomat}_{\backslash \ell}\tilde{\dv}\adjustedvbar}_{\sortedentry{\idxtarelsorted}}
\end{equation}
and therefore
\begin{equation}
	\max_{\dv\in\sphereregion(\spherec,\spherer)}\kparen{\adjustedvbar\ktranspose{\dicomat}_{\backslash \ell}\dv\adjustedvbar_{\sortedentry{\idxtarelsorted}}}
	\leq
	\kparen{\max_{\dv\in\sphereregion(\spherec,\spherer)}\adjustedvbar\ktranspose{\dicomat}_{\backslash \ell}\dv\adjustedvbar}_{\sortedentry{\idxtarelsorted}}.
\end{equation}
Combining these results leads to
\begin{equation*}
	\begin{split}
		\adjustedvbar\ktranspose{\atom}_{\column{\idxscreen}}\dvopt\adjustedvbar
		+ \sum_{\idxtarelsorted=\idximp}^{\idxps-1} \adjustedvbar\ktranspose{\dicomat}_{\backslash \idxscreen}\dvopt\adjustedvbar_{\sortedentry{\idxtarelsorted}}
		&\leq \max_{\dv\in\sphereregion(\spherec,\spherer)}
		\kparen{
		\adjustedvbar\ktranspose{\atom}_{\column{\idxscreen}}\dv\adjustedvbar
		+ \sum_{\idxtarelsorted=\idximp}^{\idxps-1} \adjustedvbar\ktranspose{\dicomat}_{\backslash \idxscreen}\dv\adjustedvbar_{\sortedentry{\idxtarelsorted}}
		}\\
		&\leq \max_{\dv\in\sphereregion(\spherec,\spherer)}
		\adjustedvbar\ktranspose{\atom}_{\column{\idxscreen}}\dv\adjustedvbar
		+ \sum_{\idxtarelsorted=\idximp}^{\idxps-1} \max_{\dv\in\sphereregion(\spherec,\spherer)}  \kparen{\adjustedvbar\ktranspose{\dicomat}_{\backslash \idxscreen}\dv\adjustedvbar_{\sortedentry{\idxtarelsorted}}
		}\\
		&\leq \max_{\dv\in\sphereregion(\spherec,\spherer)}
		\adjustedvbar\ktranspose{\atom}_{\column{\idxscreen}}\dv\adjustedvbar
		+ \sum_{\idxtarelsorted=\idximp}^{\idxps-1}  \kparen{\max_{\dv\in\sphereregion(\spherec,\spherer)}\adjustedvbar\ktranspose{\dicomat}_{\backslash \idxscreen}\dv\adjustedvbar}_{\sortedentry{\idxtarelsorted}
		}\\
		&\leq
		\adjustedvbar{
		\ktranspose{\atom}_{\column{\idxscreen}}\spherec
		}\adjustedvbar
		+ \sum_{\idxtarelsorted=\idximp}^{\idxps-1} \adjustedvbar{
			\ktranspose{\dicomat}_{\backslash \idxscreen}\spherec
		}\adjustedvbar_{\sortedentry{\idxtarelsorted}}
		+
		(\idxps-\idximp+1)
		\spherer.
	\end{split}
\end{equation*}

\subsection[Proof of Lemma~\ref{lemma:optimality p=0 for LASSO}]{Proof of \Cref{lemma:optimality p=0 for LASSO}}
\label{sec:Proof of lemma:optimality p=0 for LASSO}

We want to show that if test \eqref{eq: general safe screening for SLOPE} is passed for some \(\{\idximp_\idxps\}_{\idxps\in\intervint{1}{\tarvecdim}}\), then test \eqref{eq: safe screening for SLOPE p=0} is also passed when \(\slopeweight_\idxtarelsorted=1\) \(\forall\idxtarelsorted\in\intervint{1}{\tarvecdim}\).

Assume \eqref{eq: general safe screening for SLOPE} holds for some \(\{\idximp_\idxps\}_{\idxps\in\intervint{1}{\tarvecdim}}\), that is \(\forall\idxps\in\intervint{1}{\tarvecdim}\), \(\exists \idximp_\idxps\in\intervint{1}{\idxps}\) such that
\begin{equation}
	\label{eq:buf:inquality proof2 dirty}
	\adjustedvbar{
	\ktranspose{\atom}_{\column{\idxscreen}}\spherec
	}\adjustedvbar
	+ \sum_{\idxtarelsorted=\idximp_\idxps}^{\idxps-1} \adjustedvbar{
		\ktranspose{\dicomat}_{\backslash \idxscreen}\spherec
	}\adjustedvbar_{\sortedentry{\idxtarelsorted}}
	<
	\screenthres_{\idxps,\idximp_\idxps}
	,
\end{equation}
where
\(\screenthres_{\idxps,\idximp} \triangleq
\lambda
\kparen{\sum_{\idxtarelsorted=\idximp}^{\idxps} \slopeweight_\idxtarelsorted}
-
(\idxps-\idximp+1)\spherer
\).
Considering the case ``\(\idxps=1\)'', we have $\idximp_1=1$, $\screenthres_{1,1}=\lambda \slopeweight_1-\spherer$ and \eqref{eq:buf:inquality proof2 dirty} thus particularizes to
\begin{equation}
	\adjustedvbar{
	\ktranspose{\atom}_{\column{\idxscreen}}\spherec
	}\adjustedvbar
	<\lambda \slopeweight_1 - \spherer
	.
\end{equation}
Since \(\slopeweight_\idxtarelsorted=1\) \(\forall\idxtarelsorted\in\intervint{1}{\tarvecdim}\) by hypothesis, the latter inequality is equal to \eqref{eq: safe screening for SLOPE p=0} and the result is proved.\\

\subsection[Proof of Lemma~\ref{lemma:nesting test}]{Proof of \Cref{lemma:nesting test}}
\label{subsec:proof lemma nesting test}

\newcommand{\constantsymb}{C}
\newcommand{\innerprodsymb}{t}
\newcommand{\cumsumsymb}{\sigma}
We prove the result by showing that \(\forall\idxps\in\intervint{1}{\tarvecdim}\) the sequence \(\{\upperbound_{\idxps,\idxscreen}\}_{\idxscreen\in\intervint{1}{\tarvecdim}}\) is non-increasing.
To this end, we first rewrite \(\upperbound_{\idxps,\idxscreen}\) in a slightly different manner, easier to analyze.
Let
\begin{equation}
	\begin{array}{rll}
		\constantsymb_{\idxps,\idximp}
		\triangleq &
		(\idxps-\idximp+1)
		\spherer
		+ \lambda\kparen{\sum_{\idxtarelsorted=1}^{\idximp-1} \slopeweight_{\idxtarelsorted}}
		           & \forall\idxps\in\intervint{1}{\tarvecdim}, \forall\idximp\in\intervint{1}{\idxps}        \\
		\cumsumsymb_\idxps
		\triangleq
		           & \sum_{\idxtarelsorted=1}^\idxps |\ktranspose{\atom}_{\column{\idxtarelsorted}} \spherec|
		           & \forall\idxps\in\intervint{0}{\tarvecdim}
	\end{array}
\end{equation}
with the convention \(\cumsumsymb_0\triangleq0\).
Using these notations and hypothesis~\eqref{eq:WH}, \(\upperbound_{\idxps,\idxscreen}\) can be rewritten as
\begin{align}
	\upperbound_{\idxps,\idxscreen} - \constantsymb_{\idxps,\idximp}
	\;=\; &
	\adjustedvbar\ktranspose{\atom}_{\column{\idxscreen}} \spherec\adjustedvbar
	+ \sum_{\idxtarelsorted=1}^{\idxps-1} \adjustedvbar{
		\ktranspose{\dicomat}_{\setminus \idxscreen}\spherec
	}\adjustedvbar_{\entry{\idxtarelsorted}}
	- \sum_{\idxtarelsorted=1}^{\idximp-1} \adjustedvbar{
		\ktranspose{\dicomat}_{\setminus \idxscreen}\spherec
	}\adjustedvbar_{\entry{\idxtarelsorted}}
	\\
	\;=\; &
	\begin{cases}
		|\ktranspose{\atom}_{\column{\idxscreen}} \spherec| + \cumsumsymb_{\idxps-1} - \cumsumsymb_{\idximp-1} & \text{ if } \idxps < \idxscreen                  \\
		\cumsumsymb_{\idxps} - \cumsumsymb_{\idximp-1}                                                         & \text{ if } \idximp - 1 < \idxscreen \leq \idxps \\
		|\ktranspose{\atom}_{\column{\idxscreen}} \spherec| + \cumsumsymb_{\idxps} - \cumsumsymb_{\idximp}     & \text{ if } \idxscreen \leq \idximp - 1.
	\end{cases}
	\label{eq:sqlp}
\end{align}
We next show that \(\forall \idxps\in\intervint{1}{\tarvecdim}\) the sequence \(\{\upperbound_{\idxps,\idxscreen}\}_{\idxscreen\in\intervint{1}{\tarvecdim}}\) is non-increasing. We first notice that \(\constantsymb_{\idxps,\idximp}\) does not depend on \(\idxscreen\) and we can therefore focus on \eqref{eq:sqlp} to prove our claim.
Using the fact that \(|\ktranspose{\atom}_{\column{\idxscreen}} \spherec| \geq |\ktranspose{\atom}_{\column{\idxscreen+1}} \spherec|\) by hypothesis, we immediately obtain that \(\upperbound_{\idxps,\idxscreen} \geq \upperbound_{\idxps,\idxscreen+1}\) whenever \(\idxscreen\notin\{\idximp-1,\idxps\}\).
We conclude the proof by treating the cases ``\(\idxscreen=\idximp-1\)'' and ``\(\idxscreen=\idxps\)'' separately.
If \(\idxscreen=\idxps\) we have from~\eqref{eq:sqlp}:
\begin{equation}
	\upperbound_{\idxps,\idxscreen+1} - \upperbound_{\idxps,\idxscreen}
	=
	|\ktranspose{\atom}_{\column{\idxps+1}} \spherec| + \cumsumsymb_{\idxps-1} - \cumsumsymb_{\idxps}
	=
	|\ktranspose{\atom}_{\column{\idxps+1}} \spherec| - |\ktranspose{\atom}_{\column{\idxps}} \spherec| \leq 0
	,
\end{equation}
where the last inequality holds true by virtue of~\eqref{eq:WH}.

If \(\idxscreen=\idximp-1\) (and provided that \(\idximp\geq2\)) the same rationale leads to
\begin{equation}
	\upperbound_{\idxps,\idxscreen+1} - \upperbound_{\idxps,\idxscreen}
	=
	|\ktranspose{\atom}_{\column{\idximp}} \spherec|  - |\ktranspose{\atom}_{\column{\idximp-1}} \spherec| \leq 0
	.
\end{equation}
\vspace{0.05cm}

\subsection[Proof of Lemma~\ref{lemma: subset inequality}]{Proof of \Cref{lemma: subset inequality}}\label{proof:lemma: subset inequality}

The necessity of \eqref{eq: threshold-based screening test} can be shown as follows. Assume
\(\ |\ktranspose{\atom}_{\column{\tarvecdim}}\spherec|\geq\thresinequ\) for some \(\thresinequ\in\setinequ\) and let \(\idxps\in\intervint{1}{\tarvecdim}\) be such that \(\thresinequ=\thresinequ_{\idxps,\idximp^\star(\idxps)}\).
From \eqref{eq: meaning r*} we then have
\begin{equation}
	\forall\idximp\in\intervint{1}{\idxps}:\
	|{
	\ktranspose{\atom}_{\column{\tarvecdim}}\spherec
	}|\geq\thresinequ_{\idxps,\idximp}
\end{equation}
and test~\eqref{eq: screening test compressed form} therefore fails.

To prove the sufficiency of \eqref{eq: threshold-based screening test},
let us first notice that the definition of \(\thresinequ_{\idxps,\idximp}\) given in~\eqref{eq:def tau 2} can be naturally extended to any arbitrary couple of indices \(\idxps,\idximp\in\intervint{1}{\pvdim}\), \ie
\begin{equation} \label{eq:def tau 2 extended}
	\forall \idxps,\idximp\in\intervint{1}{\pvdim}:\quad
	\thresinequ_{\idxps,\idximp}
	= \screeningboundfunc(\idximp) - (\screeningboundfunc(\idxps) - \lambda \slopeweight_\idxps) -\spherer
	.
\end{equation}
On the other hand, the index \(\idxps^{(1)}\) has been defined as
\begin{equation} \label{eq:recall def x^(1)}
	\idxps^{(1)}
	\triangleq \idxps^\star(\pvdim)
	=
	\kargmax_{\idxps\in\intervint{1}{\pvdim}} \screeningboundfunc(\idxps) - \lambda \slopeweight_\idxps
	,
\end{equation}
see~\eqref{eq:def q^star(k)} and~\eqref{eq:def q set}.
Combining~\eqref{eq:def tau 2 extended} and~\eqref{eq:recall def x^(1)}, one obtains \(\forall \idximp\in \intervint{1}{\tarvecdim}\):
\begin{equation}
	\thresinequ_{\idxps^{(1)},\idximp}
	=
	\kargmin_{\idxps\in \intervint{1}{\pvdim}} \thresinequ_{\idxps,\idximp}
	.
\end{equation}
In particular, letting \(\idximp=\idximp^{(1)}\), we have
\begin{equation}
	\label{eq:inequality idximp^(1)}
	\forall \idxps\in\intervint{\idximp^{(1)}}{\tarvecdim}:\ \thresinequ_{\idxps^{(1)},\idximp^{(1)}}\leq\thresinequ_{\idxps,\idximp^{(1)}}
	.
\end{equation}
Hence,
\begin{equation}\label{eq: partial screening test}
	|{
	\ktranspose{\atom}_{\column{\tarvecdim}}\spherec
	}|
	<
	\thresinequ_{\idxps^{(1)},\idximp^{(1)}}
	\implies
	\forall \idxps\in\intervint{\idximp^{(1)}}{\tarvecdim}:\
	|{
	\ktranspose{\atom}_{\column{\tarvecdim}}\spherec
	}|
	<
	\thresinequ_{\idxps,\idximp^{(1)}}
	.
\end{equation}
In other words, satisfying the left-hand side of \eqref{eq: partial screening test} implies that test \eqref{eq: screening test compressed form} is verified for each \(\idxps\in\intervint{\idximp^{(1)}}{\tarvecdim}\).

We can apply the same reasoning iteratively to show that \(\forall \idxinequ\in\intervint{1}{\card[\setinequ]}\):
\begin{equation}
	|{
	\ktranspose{\atom}_{\column{\tarvecdim}}\spherec
	}|
	<
	\thresinequ_{\idxps^{(\idxinequ)},\idximp^{(\idxinequ)}}
	\implies
	\forall \idxps\in\intervint{\idximp^{(\idxinequ)}}{\idximp^{(\idxinequ-1)}-1}:\
	|{
	\ktranspose{\atom}_{\column{\tarvecdim}}\spherec
	}|
	<
	\thresinequ_{\idxps,\idximp^{(\idxinequ)}}
	.
\end{equation}
Since \(\idximp^{(\card[\setinequ])}=1\), we obtain that \eqref{eq: threshold-based screening test} implies that \eqref{eq: screening test compressed form} is verified \(\forall\idxps\in\intervint{1}{\tarvecdim}\).\\


\clearpage
\bibliographystyle{siamplain}
\bibliography{references}
\end{document}